\PassOptionsToPackage{hypertexnames=false}{hyperref}
\documentclass[11pt]{article}
\usepackage{fullpage}
\usepackage[hyphens]{url}
\usepackage{amsmath,amssymb,amsthm,mathtools}
\usepackage{enumitem}
\usepackage{microtype}
\usepackage{xcolor}
\usepackage[algo2e,ruled]{algorithm2e}
\SetAlgoNoEnd
\usepackage{array}
\usepackage{booktabs}
\usepackage{multirow}
\usepackage{thmtools}
\usepackage{thm-restate}
\usepackage[colorlinks=true,linkcolor=blue!45!black,citecolor=blue!45!black,urlcolor=blue!55!black]{hyperref}
\usepackage[nameinlink,capitalize,noabbrev]{cleveref}
\usepackage{autonum}
\makeatletter
\let\forlistloop\autonum@etoolbox@forlistloop
\def\blx@noerroretextools{}
\makeatother
\usepackage[
  date=year,
  eprint=false,
  doi=false,
  isbn=false,
  backend=biber,
  giveninits=true,
  uniquename=init,
  maxcitenames=2,
  maxbibnames=99,
  natbib=true,
  url=false,
  sorting=ynt,
  style=apa,
  citestyle=authoryear-comp,
  backref=true,
  uniquelist=false
]{biblatex}

\DeclareSourcemap{
  \maps[datatype=bibtex]{
    \map{
      \step[fieldset=editor, null]
      \step[fieldset=language, null]
      \step[fieldset=address, null]
      \step[fieldset=location, null]
      \step[fieldset=month, null]
      \step[fieldset=annote, null]
    }
  }
}
\DefineBibliographyStrings{english}{%
  backrefpage  = {cited on page},
  backrefpages = {cited on pages},
}
\DefineBibliographyStrings{american}{%
  backrefpage  = {cited on page},
  backrefpages = {cited on pages},
}
\DefineBibliographyStrings{american-apa}{%
  backrefpage  = {cited on page},
  backrefpages = {cited on pages},
}
\renewbibmacro*{pageref}{%
  \iflistundef{pageref}
    {}
    {\printtext[parens]{%
       \midsentence%
       \ifnumgreater{\value{pageref}}{1}
         {\bibstring{backrefpages}\ppspace}
         {\bibstring{backrefpage}\ppspace}%
       \printlist[pageref][-\value{listtotal}]{pageref}}}}
\hypersetup{
  pdftitle={Tight Regret Bound for Online Inverse Linear Optimization via Multiscale Matrix Weights}
}
\newtheorem{theorem}{Theorem}[section]
\newtheorem{lemma}[theorem]{Lemma}

\crefname{lemma}{Lemma}{Lemmas}
\Crefname{lemma}{Lemma}{Lemmas}
\crefname{proposition}{Proposition}{Propositions}
\Crefname{proposition}{Proposition}{Propositions}
\crefname{theorem}{Theorem}{Theorems}
\Crefname{theorem}{Theorem}{Theorems}
\crefname{algorithm}{Algorithm}{Algorithms}
\crefname{algocf}{Algorithm}{Algorithms}
\crefname{section}{Section}{Sections}
\Crefname{section}{Section}{Sections}
\crefname{appendix}{Appendix}{Appendices}
\Crefname{appendix}{Appendix}{Appendices}
\crefname{table}{Table}{Tables}
\Crefname{table}{Table}{Tables}
\makeatletter
\g@addto@macro\appendix{\crefalias{section}{appendix}\crefalias{subsection}{appendix}}
\makeatother
\newcommand{\savednormallabel}{}
\AtBeginDocument{\global\let\savednormallabel\label}
\newcommand{\restorenormallabel}{\global\let\label\savednormallabel}
\makeatletter
\newcommand{\wraprestatablewithlabelrestore}[1]{%
  \csletcs{restatable@orig@#1}{#1}%
  \csdef{#1}{%
    \@ifstar{%
      \csuse{restatable@orig@#1}*%
      \restorenormallabel
    }{%
      \csuse{restatable@orig@#1}%
    }%
  }%
}
\makeatother
\DeclarePairedDelimiter{\prn}{(}{)}
\DeclarePairedDelimiter{\set}{\{}{\}}
\DeclarePairedDelimiterX{\Set}[2]{\{}{\}}{\,{#1}\,:\,{#2}\,}
\DeclarePairedDelimiter{\abs}{|}{|}
\DeclarePairedDelimiter{\norm}{\|}{\|}
\DeclarePairedDelimiter{\inpr}{\langle}{\rangle}
\DeclarePairedDelimiter{\brc}{[}{]}
\DeclarePairedDelimiter{\ceil}{\lceil}{\rceil}
\DeclarePairedDelimiter{\floor}{\lfloor}{\rfloor}
\newcommand{\R}{\mathbb{R}}
\newcommand{\Q}{\mathbb{Q}}
\newcommand{\E}{\mathbb{E}}
\newcommand{\B}{\mathbb{B}}
\newcommand{\conv}{\operatorname{conv}}
\newcommand{\BR}{\operatorname{BR}}
\newcommand{\Reg}{\operatorname{Reg}}
\newcommand{\poly}{\operatorname{poly}}
\newcommand{\eps}{\varepsilon}
\DeclareMathOperator*{\argmax}{arg\,max}
\newcommand{\Tr}{\operatorname{Tr}}
\newcommand{\op}{\mathrm{op}}
\numberwithin{equation}{section}
\allowdisplaybreaks
\title{Tight Regret Bound for Online Inverse Linear Optimization via Multiscale Matrix Weights\thanks{In view of rapid recent progress in this area
\citep{Dewasurendra2026,Kitaoka2026SGS,CaiEtAl2026}, the author is sharing this
version before its exposition has reached the intended level of polish.
The mathematical arguments are complete. The existence guarantees in \cref{thm:regret,thm:computation}
have been verified in Lean with AI assistance, assuming standard external theorems. 
The author plans to expand the content and improve the presentation.}}
\author{
  Shinsaku Sakaue%
\thanks{CyberAgent, Tokyo, Japan; National Institute of Informatics, Tokyo, Japan; Center for Advanced Intelligence Project, RIKEN, Tokyo, Japan. Email: \texttt{shinsaku.sakaue@gmail.com}.}%
}
\date{}

\begin{document}
\maketitle
\begin{abstract}
We study online inverse linear optimization with a fixed unknown linear utility: in
each round, an environment presents a compact action set, the learner recommends an
action from it, and the environment returns an action that maximizes the utility over
the same set.  When
the utility vector and the actions lie in the $d$-dimensional Euclidean unit ball, we give a
randomized algorithm whose \emph{regret}---the cumulative utility shortfall
relative to optimal actions---is $O(\sqrt d)$ in expectation for every time horizon,
without knowledge of the horizon.  The dependence on $d$ is optimal up to
a constant factor by the known $\Omega(\sqrt d)$ lower bound for horizons $T\ge d$.
Our algorithm maintains matrix multiplicative weights on polynomial feature spaces at
geometrically spaced scales. It selects a recommendation distribution by solving a linear
program and updates its score matrices by comparing the available actions with the
feedback action.
With rational oracle outputs and feedback actions, an implementation computable
relative to a linear-optimization oracle preserves the $O(\sqrt d)$ regret bound. Whether the same rate is attainable
with running time polynomial in the dimension, horizon, and input length remains open.
\end{abstract}

\section{Introduction}
\label{sec:introduction}

Learning from observed decisions rather than numerical objective values is
central to inverse optimization \citep{ChanEtAl2025}, which has applications
to control \citep{AkhtarEtAl2022}, routing
\citep{ZattoniScroccaroEtAl2025Routing}, and reinforcement learning
\citep{DimanidisEtAl2025}.
We study online inverse linear optimization. In each round, an environment
presents an action set, the learner recommends an action from it, and the
environment then returns an action that maximizes a fixed unknown linear
utility over the same set.
The feedback identifies an optimal action but reveals no numerical utility values.
Here, we ask how the cumulative utility shortfall of the recommendations,
or \emph{regret},
must grow with the dimension $d$ when the number of rounds is unrestricted.

\begin{table}[!t]
\small
\caption{Expected regret bounds for online inverse linear optimization with utility vector and actions
in the Euclidean unit ball, for $T\ge2$. ``Uniform in $T$'' means a $T$-independent
bound, even if the algorithm needs $T$. A proper learner fixes a nonzero utility estimate before
observing the current action set and recommends a maximizer of that estimate.
For the proper implementations listed here, a zero internal estimate is handled by
maximizing a fixed nonzero direction when selecting the action.
LO denotes one linear-optimization oracle call over the current action set; counts
are per-round upper bounds, with additional work listed separately.
The $O(d)$ and $O(d^2)$ terms in the per-round computation count arithmetic operations; the
$\poly(d,T)$ entries are the geometric implementation bounds stated in the cited
works, rather than bounds in a common bit-complexity model.
The gradient-descent cost includes Euclidean projection onto the utility unit ball.
Dewasurendra's LO count is for the finite known-horizon construction; vote maximization
requires access beyond LO.  Our rational implementation uses the grid in
\cref{app:rational-actions}, with no proved polynomial bound on the remaining
computation.  The lower bound in \cref{prop:lower} (for $T\ge d$) uses the construction of
\citet[Theorem~5.1]{SakaueEtAl2025NeurIPS}; \citet[Theorem~3.1]{CaiEtAl2026} give a related bound for a cutting-plane game.\looseness=-1}
\label{tab:comparison}
\smallskip
\centering
\setlength{\tabcolsep}{3pt}
\renewcommand{\arraystretch}{1.15}
\begin{tabular}{@{}>{\raggedright\arraybackslash}p{2.05in}l@{\hspace{-4pt}}c@{\hspace{13pt}}c@{\hspace{9pt}}>{\raggedright\arraybackslash}p{1.75in}@{}}
\toprule
Method & Regret & Uniform in $T$ & Proper & Per-round computation \\
\midrule
Online gradient descent
 & \multirow{2}{*}{$O(\sqrt T)$} & \multirow{2}{*}{no} & \multirow{2}{*}{yes}
 & \multirow{2}{=}{1 LO + $O(d)$} \\
\citep{BaermannEtAl2020} & & & & \\
\addlinespace[4pt]
ProjectedCones ($d\ge2$)
 & \multirow{2}{*}{$O(d^4\log T)$} & \multirow{2}{*}{no} & \multirow{2}{*}{yes}
 & \multirow{2}{=}{1 LO + $\poly(d,T)$} \\
\citep{BesbesEtAl2025} & & & & \\
\addlinespace[4pt]
Centroid cutting plane
 & \multirow{2}{*}{$O(d\log T)$} & \multirow{2}{*}{no} & \multirow{2}{*}{yes}
 & \multirow{2}{=}{1 LO + $\poly(d,T)$} \\
\citep{GollapudiEtAl2021NeurIPS} & & & & \\
\addlinespace[4pt]
John-ellipsoid cutting plane
 & \multirow{2}{*}{$\exp(O(d\log d))$} & \multirow{2}{*}{yes} & \multirow{2}{*}{yes}
 & \multirow{2}{=}{1 LO + $\poly(d,T)$} \\
\citep{GollapudiEtAl2021NeurIPS} & & & & \\
\addlinespace[4pt]
Online Newton step
 & \multirow{2}{*}{$O(d\log T)$} & \multirow{2}{*}{no} & \multirow{2}{*}{yes}
 & \multirow{2}{=}{1 LO + 1 Mahalanobis projection + $O(d^2)$} \\
\citep{SakaueEtAl2025NeurIPS} & & & & \\
\addlinespace[4pt]
Second-order perceptron
 & \multirow{2}{*}{$O(d\log T)$} & \multirow{2}{*}{no} & \multirow{2}{*}{yes}
 & \multirow{2}{=}{1 LO + $O(d^2)$} \\
\citep{Sakaue2026ProjectionFree} & & & & \\
\addlinespace[4pt]
Multiscale reward hedging
 & \multirow{2}{*}{$O(d)$} & \multirow{2}{*}{yes} & \multirow{2}{*}{no}
 & \multirow{2}{=}{$T^{O(d)}$ LO + nonconvex vote maximization} \\
\citep{Dewasurendra2026} & & & & \\
\addlinespace[4pt]
Variable-metric algorithm
 & \multirow{2}{*}{$O(d)$} & \multirow{2}{*}{yes} & \multirow{2}{*}{yes}
 & \multirow{2}{=}{1 LO + $O(d^2)$} \\
\citep{CaiEtAl2026} & & & & \\
\addlinespace[4pt]
\textbf{Multiscale matrix weights}
 & \multirow{2}{*}{$O(\sqrt d)$}
 & \multirow{2}{*}{yes} & \multirow{2}{*}{no}
 & \multirow{2}{=}{$(dT)^{O(d)}$ LO + finite computation} \\
\textbf{(this paper)} & & & & \\
\midrule
Lower bound
(\cref{prop:lower})
& $\Omega(\sqrt d)$ & --- & --- & --- \\
\bottomrule
\end{tabular}
\end{table}

We answer this question up to a universal constant when the utility vector and the actions
have Euclidean norm at most one.  We give a randomized algorithm whose expected regret
is at most $2^{21}\sqrt d$ for every horizon, without knowledge of the horizon.
A lower bound of $\sqrt d$ for every learner and every $T\ge d$ follows from the
construction of \citet[Theorem~5.1]{SakaueEtAl2025NeurIPS}, whereas the best previous upper bounds
under this normalization are $O(d)$ \citep{Dewasurendra2026,CaiEtAl2026}.  Our bound
closes this gap: the minimax expected regret is of order $\sqrt d$ for every horizon
$T\ge d$.  The action sets may be arbitrary nonempty compact sets that depend on the
past interaction, and ties among optimal actions may be broken after the recommendation
law and its sampled action are revealed.  \Cref{tab:comparison} places this guarantee
among previous regret bounds for the same problem.

Our algorithm also admits a computable implementation.  It accesses each action set
only through a linear-optimization oracle, which returns, for a query vector
$q\in\R^d$, an action maximizing $\inpr{q,\cdot}$ over the current action set.  When
the oracle's outputs and the feedback actions have rational coordinates, the
implementation attains the same regret bound, and each of its rounds terminates for
every realization of its internal randomness; \cref{app:computation} specifies the
oracle model and constructs the implementation.
The main technical idea is to represent action comparisons by matrices on polynomial
features and combine their updates across multiple accuracy scales;
\cref{sec:overview} gives an overview of how this yields the regret bound.

\subsection{Related Work}

\paragraph{Online Inverse Optimization.}
\citet{BaermannEtAl2017,BaermannEtAl2020} formulate inverse linear optimization as
online learning and obtain regret growing as the square root of the horizon.
\citet[Theorem~4]{BesbesEtAl2025} use the ProjectedCones algorithm to obtain an
$O(d^4\log T)$ regret bound for $d\ge2$.  For the same action-feedback protocol, termed contextual
recommendation, \citet{GollapudiEtAl2021NeurIPS} give cutting-plane algorithms with
$O(d\log T)$ regret and, separately, $\exp(O(d\log d))$ horizon-independent regret.
\citet{SakaueEtAl2025NeurIPS} obtain $O(d\log T)$ regret with an online Newton step,
and \citet{Sakaue2026ProjectionFree} retains this order while removing the
Mahalanobis projection from the update.

For unrestricted compact action sets, \citet{Dewasurendra2026} obtains $O(d)$
regret by combining weighted optimality tests across accuracy scales.  This
construction applies more generally to classes of reward functions through their
covering numbers, but maximizing the resulting voting score need not be computationally efficient.
\citet{CaiEtAl2026} achieve deterministic $O(d)$ regret with one linear
optimization and $O(d^2)$ arithmetic operations per round, using a self-normalized
variable-metric update.  Their learner is \emph{proper}: it fixes a nonzero utility
estimate before observing the current action set and recommends a maximizer of that
estimate.
Our learner, like Dewasurendra's, is \emph{improper}: its recommendation need not
maximize a nonzero utility estimate fixed before the current action set is observed.
Thus our improvement concerns the dimension dependence of expected regret
(\cref{tab:comparison}); our
rational-oracle implementation guarantees termination, not polynomial running time.

Other horizon-independent regret guarantees exploit stronger structural assumptions.
\citet[Definition~5.1 and Theorem~5.2]{SakaueEtAl2025AISTATS} require
each action's objective gap to be at least a fixed positive multiple of its
distance from the optimal action, while \citet[Theorem~4.2]{OkiSakaue2026} exploit the exchange structure of
M-convex action sets, including matroid bases.  \citet{Kitaoka2026SGS}
obtains finite regret and finitely many mistakes for the online Newton step on
uniformly bounded integer action sets,\footnote{In each round,
the algorithm uses one linear optimization returning an extreme optimal action,
$O(d^2)$ arithmetic operations, and at most one Mahalanobis projection.} with utilities in the probability simplex
and unique optimal actions. The regret bound has a $d^2$ factor and additional
coordinate-range dependence.  Our
$O(\sqrt d)$ expected-regret bound instead permits arbitrary compact action sets
and ties, without a positive-margin assumption or a common integer-lattice assumption. It does not
assert a finite number of suboptimal recommendations.

Suboptimal action feedback has also been studied under several loss and corruption
models \citep{BaermannEtAl2020,SakaueEtAl2025NeurIPS,Sakaue2026ProjectionFree,OkiSakaue2026,Dewasurendra2026}.
\citet{CaiEtAl2026} bound regret relative to the returned actions whenever each is at
least as good as the learner's recommendation; they also give corruption-robust and
rank-adaptive guarantees.  Our analysis assumes exactly optimal feedback and
addresses the worst-case dependence on the ambient dimension, leaving the extension to suboptimal feedback for future work.

\paragraph{Offline Inverse Optimization.}
Classical inverse optimization adjusts objective coefficients to make given
decisions optimal, as in inverse shortest paths \citep{BurtonToint1992} and inverse
linear programming \citep{AhujaOrlin2001}.  Data-driven formulations estimate
objectives from multiple observations through optimality-condition residuals
\citep{KeshavarzEtAl2011}, suboptimality losses under imperfect information
\citep{MohajerinEsfahaniEtAl2018}, or incenter and augmented-suboptimality
formulations \citep{ZattoniScroccaroEtAl2025Learning}.
\citet{BertsimasEtAl2015} infer equilibrium models through inverse variational
inequalities.
Other work establishes statistical consistency with noisy observations
\citep{AswaniEtAl2018}, learns feasible regions with known objectives
\citep{RenEtAl2025}, or bounds generalization error for consistent estimators from
noiseless data \citep{FatemiEtAl2026}.  In an online setting,
\citet{FatemiEtAl2026} also give a cumulative-regret bound under i.i.d.\ contexts,
whereas our guarantee permits adaptively chosen action sets.  \citet{ChanEtAl2025}
survey the broader methodological and application literature.

\paragraph{Contextual Search and Binary Feedback.}
Contextual search predicts the value of a hidden linear function along a direction
revealed before the prediction and receives only an above/below comparison between the prediction and the hidden value
\citep{LobelEtAl2018}.  Its guarantees concern cumulative prediction error or pricing
loss \citep{PaesLemeSchneider2022,LiuEtAl2021}.  Our problem is contextual
recommendation \citep{GollapudiEtAl2021NeurIPS}: the learner recommends an action $x$
and observes a feedback action $y$ from the same set. The feedback action maximizes
the fixed unknown linear utility with vector $u$, and the instantaneous regret is $\inpr{u,y-x}$.
The comparison direction $y-x$ is therefore generally unknown to the learner before the recommendation.  In one-bit compressed sensing, \citet{PlanVershynin2013}
reconstruct the direction of a sparse vector from random sign measurements; their
guarantee bounds reconstruction error, whereas we bound cumulative regret without
sparsity or random-context assumptions.

\paragraph{Matrix Multiplicative Weights.}
Matrix multiplicative weights and their analysis through the Golden--Thompson
inequality \citep{Golden1965,Thompson1965} are standard
\citep{AroraHazanKale2012}.  Our contribution is the construction of polynomial
comparison matrices across scales and their use in an action-feedback learner with
$O(\sqrt d)$ expected regret.

\section{Model and Main Results}
\label{sec:model}

\subsection{Protocol and Regret}
\label{sec:protocol}

Fix an integer $d\ge1$ and let $\B_2^d\coloneqq\Set{x\in\R^d}{\norm{x}_2\le1}$.  The
utility vector $u\in\B_2^d$ is fixed and unknown to the learner.  For a nonempty compact
$X\subseteq\B_2^d$ and $v\in\R^d$, define the support function and the best-response set
by
\begin{equation}
 \sigma_X(v)\coloneqq\max\Set*{\inpr{v,x}}{x\in X},
 \qquad
 \BR_X(v)\coloneqq\argmax\Set*{\inpr{v,x}}{x\in X}.
 \label{eq:support}
\end{equation}
A \emph{learner} interacts with an \emph{environment} for rounds $t=1,2,\ldots$. 
In round $t$, let $Z_t$ denote the action set, $P_t$ the recommendation law, $A_t$ the
recommended action, and $Y_t$ the feedback action; the \emph{public history} before
round $t$ is $(Z_s,P_s,A_s,Y_s)_{s<t}$.  A round proceeds in the following order.
\begin{enumerate}[leftmargin=*,itemsep=1pt]
\item The environment selects a nonempty compact $Z_t\subseteq\B_2^d$ as a function of
  the public history and reveals it.
\item The learner selects a Borel probability law $P_t$ on $Z_t$ as a function of the
  public history, $Z_t$, and private randomness, and publishes it.  It then draws
  $A_t\sim P_t$ with fresh randomness and reveals $A_t$.\looseness=-1
\item The environment selects $Y_t\in\BR_{Z_t}(u)$ as a function of the public history
  and $(Z_t,P_t,A_t)$ and reveals it.
\end{enumerate}
The environment consists of the utility $u$ and deterministic rules for the first
and third steps.  For a fixed environment, expectations are over the learner's
randomness.  For a nonempty compact $X\subseteq\B_2^d$ and $x\in X$, define
$\Delta_u(X,x)\coloneqq\sigma_X(u)-\inpr{u,x}$.
The instantaneous regret in round $t$ is
$\Delta_u(Z_t,A_t)$, and the cumulative regret after $T$
rounds is
\begin{equation}
 \Reg_T\coloneqq\sum_{t=1}^T\Delta_u(Z_t,A_t)=\sum_{t=1}^T\inpr{u,Y_t-A_t}.
 \label{eq:regret}
\end{equation}
Conditionally on the public history, $Z_t$, and $P_t$, the recommendation $A_t$ has
mean $a_t\coloneqq\int x\,P_t(dx)\in\conv(Z_t)$, where $\conv(Z_t)$ is the convex hull of $Z_t$.
Thus $\sum_{t\le T}\inpr{u,a_t-A_t}$ has expectation zero, which gives
\begin{equation}
 \E\brc{\Reg_T}=\E\brc*{\sum_{t=1}^T\inpr{u,Y_t-a_t}}.
 \label{eq:mean-regret}
\end{equation}
Although $Y_t$ may depend on $A_t$, every permitted feedback action has utility
$\sigma_{Z_t}(u)$.
We therefore bound
$\sum_{t\le T}\inpr{u,Y_t-a_t}$ on every realization of the interaction.

\subsection{Main Results}
\label{sec:main-results}
The main result of this paper is the following $O(\sqrt d)$ expected-regret bound.
\begin{restatable}{theorem}{regrettheorem}
\label{thm:regret}
There is a randomized learner, independent of the horizon, such that for every $d\ge1$,
every $u\in\B_2^d$, every environment, and every $T\ge1$, we have
$\E\brc{\Reg_T}\le2^{21}\sqrt d$.
\end{restatable}
\wraprestatablewithlabelrestore{regrettheorem}

\noindent The bound $\E\brc{\Reg_T}\le2^{21}\sqrt d$ also holds for randomized environment rules,
with expectation over both sources of randomness.  This follows by fixing the
environment's random seed, drawn independently of the learner's randomness,
applying \cref{thm:regret} to the resulting deterministic rules, and averaging over
the seed.
The constant $2^{21}$ collects the
explicit estimates of the analysis and is not optimized.
\Cref{sec:overview} outlines the construction and analysis.

Our upper bound matches the following lower bound obtained by adapting the coordinate construction of
\citet[Theorem~5.1]{SakaueEtAl2025NeurIPS} to our recommendation protocol.
\citet[Theorem~3.1]{CaiEtAl2026} give a related $\sqrt d$ lower bound for a
cutting-plane game on the Euclidean unit ball, and the $\Omega(d)$ bound of
\citet{Dewasurendra2026} under the cube normalization rescales to the same order.
\Cref{app:lower-bound} gives the proof for our protocol for completeness.

\begin{restatable}{restateproposition}{lowerboundproposition}
\label[proposition]{prop:lower}
For every randomized learner and every $d\ge1$, there are $u\in\B_2^d$ and an
environment presenting two-element action sets with a unique optimal action in every
round such that $\E\brc{\Reg_T}\ge\sqrt d$ holds for every $T\ge d$.
\end{restatable}
\wraprestatablewithlabelrestore{lowerboundproposition}

\paragraph{Oracle Access.}
The learner accesses $Z_t$ through a linear-optimization oracle, which returns
an action in $\BR_{Z_t}(q)$ for each query $q\in\R^d$.
It queries this oracle in finitely many directions to form the list of actions
from which it recommends (\cref{app:rational-actions}).
If oracle calls with rational queries terminate and return rational actions,
and the feedback actions also have rational coordinates, an implementation
computable relative to this oracle preserves the expected-regret bound of
\cref{thm:regret}.
Every round terminates for every realization of the learner's randomness.
\Cref{app:computation} specifies the oracle model and gives the implementation
and its guarantee (\cref{thm:computation}).
The oracle model in \cref{app:oracle-model} also permits a prescribed additive
optimization error; the utility vector need not be rational.

\paragraph{Notation.}
We write $e_1,\ldots,e_d$ for the standard basis of $\R^d$.
All matrices are real.  For a symmetric matrix $W$, we write $\Tr W$, $\mathrm{e}^W$,
$\lambda_{\max}(W)$, and $\norm{W}_{\op}$ for its trace, exponential, largest
eigenvalue, and operator norm, and $W\preceq V$ means that $V-W$ is positive
semidefinite. We write $I$ for the identity matrix of the appropriate size and
$r_+\coloneqq\max\set{r,0}$ for $r\in\R$. $\log$ denotes the natural logarithm.\looseness=-1

\subsection{Technical Overview}
\label{sec:overview}

In each round, the learner must control the utility shortfall
$\inpr{u,Y_t-a_t}$ of its mean recommendation $a_t$ without observing its value.
Our algorithm uses pairwise action comparisons to choose a recommendation
distribution, rather than recommending a maximizer of an estimate of $u$.
We represent these comparisons
by matrices that depend only on the actions; the unknown $u$ enters their quadratic
forms only in the analysis.

For $x,y\in\B_2^d$ and each integer $k\ge0$, we construct a symmetric
matrix $B_k(x,y)$ on the space of polynomials of degree at most $m_k\coloneqq2\cdot4^k+2$
in $d$ variables, whose dimension is $d_k\coloneqq\binom{d+m_k}{d}$
(\cref{sec:features}).
These matrices satisfy $\norm{B_k(x,y)}_{\op}\le1$ and
$B_k(y,x)=-B_k(x,y)$.  The utility enters the analysis only through a unit vector
$\phi_k(u)$ of the same space, and $B_k(x,y)$ and $\phi_k(u)$ are matched so that the quadratic form
$\phi_k(u)^\top B_k(x,y)\phi_k(u)$ vanishes when $\inpr{u,y-x}=0$ and otherwise has the
sign of $\inpr{u,y-x}$ (\cref{lem:moments}); the form is positive exactly when $y$ has
larger utility than $x$.  To quantify the utility difference, we sum over scales:
for every integer $K\ge1$, the magnitude of
$\sum_{k=0}^K2^{-k}\phi_k(u)^\top B_k(x,y)\phi_k(u)$ is comparable to
$\abs{\inpr{u,y-x}}$ up to an additive cutoff error of order $2^{-K}$
(\cref{lem:dyadic}).  Combining weighted tests
at several accuracy scales in this way follows the multiscale reward hedging of
\citet{Dewasurendra2026}; there the object at each scale is a vote over a finite cover
of the reward class, while here it is a matrix on a polynomial feature space.

Assume for now that each action set is finite; this restriction is removed in
\cref{sec:compact}.
In round~$t$, scales $k=0,\ldots,K_t$ are active, where
$K_t\coloneqq\ceil{2\log_2(t+1)}$.  At the start of the round, the learner keeps a
symmetric score matrix $W_{t,k}\in\R^{d_k\times d_k}$ for each active scale, initialized
to zero only when that scale first becomes active (\cref{sec:algorithm}).
Since $2^{-K_t}\le(t+1)^{-2}$, the cutoff errors are summable without knowing a horizon.

Let $\varrho_{t,k}\coloneqq\mathrm{e}^{W_{t,k}}/\Tr\mathrm{e}^{W_{t,k}}
\in\R^{d_k\times d_k}$.
For the current action set $Z_t=\set{x_1,\ldots,x_n}$, where $n$ is its size, the
learner computes a matrix $C\in\R^{n\times n}$ with entries
\begin{equation}
 C_{ij}\coloneqq\sum_{k=0}^{K_t}2^{-k}\Tr(\varrho_{t,k}B_k(x_i,x_j)),
 \qquad 1\le i,j\le n.
\end{equation}
The entry $C_{ij}$ compares $x_j$ against $x_i$ using the current matrices
$\varrho_{t,k}$.  The learner chooses probabilities $p_1,\ldots,p_n$ satisfying
\begin{equation}
 p_i\ge0\quad(1\le i\le n),\qquad
 \sum_{i=1}^np_i=1,\qquad
 \sum_{i=1}^np_iC_{ij}\le0\quad(1\le j\le n).
\end{equation}
The last inequalities are the \emph{balance condition}.
Such probabilities exist by the minimax theorem, since $C$ is skew-symmetric, and
are found by solving a linear program (\cref{lem:balance}).

The learner publishes the law
$P_t$ assigning probability $p_i$ to $x_i$ and draws $A_t\sim P_t$.
With $\eta\coloneqq2^{-9}$, after observing the feedback action $y=Y_t$, it updates
each active scale $0\le k\le K_t$ by
\begin{equation}
 W_{t+1,k}=W_{t,k}+\sum_{i=1}^np_i\prn*{\eta B_k(x_i,y)-\eta^2B_k(x_i,y)^2}.
\end{equation}

The analysis uses the potential \eqref{eq:potential},
$\Psi(W;\phi)\coloneqq\log\Tr \mathrm{e}^W-\phi^\top W\phi$
for a symmetric matrix $W$ and a unit vector $\phi$ of the same dimension.
It is nonnegative and, at zero initialization, satisfies
$\Psi(0;\phi_k(u))=\log d_k$.
For feedback $y=x_j$, the change in
$\sum_{k=0}^{K_t}2^{-k}\log\Tr \mathrm{e}^{W_{t,k}}$ is at most
$\eta\sum_i p_iC_{ij}\le0$; the proof uses the Golden--Thompson
inequality \citep{Golden1965,Thompson1965} and is given in \cref{lem:matrix-update}.
For optimal $y$ and mean recommendation $a_t=\sum_i p_ix_i$, the
$p_i2^{-k}$-weighted sum of the quadratic forms of $B_k(x_i,y)$ at $\phi_k(u)$ is comparable
to $\inpr{u,y-a_t}$ up to the cutoff error.
The same sum for $B_k(x_i,y)^2$ is bounded by a constant multiple of $\inpr{u,y-a_t}$,
so the fixed $\eta$ makes the first-order contribution dominate its quadratic
correction, up to the cutoff error.  Together these estimates give
\[
 \sum_{k=0}^{K_t}2^{-k}\brc*{
 \Psi(W_{t,k};\phi_k(u))-\Psi(W_{t+1,k};\phi_k(u))}
 \ge c_0\inpr{u,y-a_t}-O(2^{-K_t})
\]
for a universal constant $c_0>0$; the comparison estimates are proved in
\cref{lem:pair}.

Summing over rounds from each scale's initialization telescopes its
potential from $\log d_k$ to a nonnegative value.  Thus we have
$\sum_{t=1}^T\inpr{u,Y_t-a_t}=O(\sum_{k=0}^{K_T}2^{-k}\log d_k+1)$,
where the additive constant bounds the sum
of the cutoff errors.
The choice of $m_k$ makes
$2^{-k}\log d_k$ of order $\sqrt d$ near $2^k\approx\sqrt d$ and geometrically smaller
away from it, giving $\sum_{k\ge0}2^{-k}\log d_k=O(\sqrt d)$
(\cref{lem:initial}).  The expected regret therefore satisfies
$\E\brc{\Reg_T}=O(\sqrt d)$ by \eqref{eq:mean-regret};
\cref{sec:analysis} gives the full calculation.
For a compact action set, the learner recommends from a finite list of
actions and approximates the optimal feedback action by a convex combination of listed actions
(\cref{sec:compact}).

\section{Comparison Matrices on Polynomial Features}
\label{sec:features}

We now construct the comparison matrices and utility feature vectors
used in \cref{sec:overview}.  Because the update contains both $B_k(x,y)$ and its
square, the construction must control both quadratic forms at $\phi_k(u)$. For $x,y\in\B_2^d$ with $\inpr{u,y-x}\ge0$ and every integer
$K\ge1$, we seek bounds of the form\looseness=-1
\begin{align}
 &\sum_{k=0}^K2^{-k}\phi_k(u)^\top B_k(x,y)\phi_k(u)
 \ge c_1\inpr{u,y-x}-c_2 2^{-K}\quad\text{and}\\
 &\sum_{k=0}^K2^{-k}\phi_k(u)^\top B_k(x,y)^2\phi_k(u)
 \le c_3\inpr{u,y-x},
\end{align}
with universal constants $c_1,c_2,c_3>0$.  A sufficiently small fixed learning rate
$\eta>0$ then makes the weighted sum of the quadratic forms of
$\eta B_k(x,y)-\eta^2B_k(x,y)^2$ at $\phi_k(u)$ at least a positive constant times
the utility difference, up to the truncation error, even for arbitrarily small
positive differences.
We choose the features and degrees to meet these requirements;
their quantitative bounds are proved in \cref{lem:dyadic,lem:pair}.

\subsection{Features and Comparison Matrices}
\label{sec:feature-construction}

For $k\ge0$, the scale, the polynomial degree, and the feature dimension are
\begin{equation}
 \alpha_k\coloneqq2^k,\qquad
 m_k\coloneqq2\alpha_k^2+2,\qquad
 d_k\coloneqq\binom{d+m_k}{d}.
 \label{eq:scales}
\end{equation}
For an integer $m\ge2$, the feature space of degree at most $m$ is $\R^{\mathcal I_m}$ with
$\mathcal I_m\coloneqq\Set{\beta\in\mathbb N_0^d}{\abs{\beta}\le m}$, where
$\abs{\beta}\coloneqq\sum_{\ell=1}^d\beta_\ell$; its standard basis is
$(e_\beta)_{\beta\in\mathcal I_m}$, and
$\abs{\mathcal I_{m_k}}=d_k$.  For $\beta\in\mathcal I_m$ and $u\in\R^d$, write
$\beta!\coloneqq\prod_{\ell=1}^d\beta_\ell!$ and
$u^\beta\coloneqq\prod_{\ell=1}^du_\ell^{\beta_\ell}$, with $u_\ell^0=1$, and let
$\operatorname{exp}_m(s)\coloneqq\sum_{j=0}^ms^j/j!$ denote the degree-$m$ Taylor
polynomial of the exponential about zero, evaluated at $s\in\R$.  The feature vector of the utility $u$ at
scale $k$ is
\begin{equation}
 \phi_k(u)_\beta\coloneqq
 \frac{\alpha_k^{\abs{\beta}}u^\beta}
 {\sqrt{\beta!\,\operatorname{exp}_{m_k}(\alpha_k^2\norm{u}_2^2)}},
 \qquad \beta\in\mathcal I_{m_k}.
 \label{eq:features}
\end{equation}
By the multinomial identity $\sum_{\abs{\beta}=j}u^{2\beta}/\beta!=\norm{u}_2^{2j}/j!$, it
is a unit vector.

The comparison matrices are built from one symmetric operator per unit direction of
$\R^d$.  For the direction $e_1$, write a multiindex as $\beta=(\beta_1,\beta')$ with
$\beta'=(\beta_2,\ldots,\beta_d)\in\mathbb N_0^{d-1}$, where $\beta'$ is the empty tuple when $d=1$, and let
$H_m(e_1)\in\R^{\mathcal I_m\times\mathcal I_m}$ be the matrix defined by
\begin{equation}
 H_m(e_1)e_{(1,\beta')}=e_{(2,\beta')},\qquad
 H_m(e_1)e_{(2,\beta')}=e_{(1,\beta')}
 \quad(\abs{\beta'}\le m-2)
 \label{eq:swap}
\end{equation}
for all such $\beta'$.  It maps every basis vector outside the union of
these pairs to zero.  In the ordered basis $(e_{(1,\beta')},e_{(2,\beta')})$ of each
paired coordinate plane, its block is
$\begin{psmallmatrix}0&1\\1&0\end{psmallmatrix}$.
Thus $H_m(e_1)$ is symmetric with operator norm one, and $H_m(e_1)^2$ fixes
every basis vector belonging to one of these pairs.

For $m=m_k$ and $u\in\B_2^d$, consider the quadratic forms
\begin{equation}
 \phi_k(u)^\top H_{m_k}(e_1)\phi_k(u),\qquad
 \phi_k(u)^\top H_{m_k}(e_1)^2\phi_k(u).
\end{equation}
For each $\beta'\in\mathbb N_0^{d-1}$ with $\abs{\beta'}\le m_k-2$,
\eqref{eq:features} gives
\begin{equation}
 \begin{pmatrix}\phi_k(u)_{(1,\beta')}\\\phi_k(u)_{(2,\beta')}\end{pmatrix}
 =\frac{1}
      {\sqrt{\operatorname{exp}_{m_k}(\alpha_k^2\norm{u}_2^2)}}
      \displaystyle\prod_{\ell=2}^d
       \frac{(\alpha_k u_\ell)^{\beta_\ell}}{\sqrt{\beta_\ell!}}
   \begin{pmatrix}\alpha_k u_1\\(\alpha_k u_1)^2/\sqrt2\end{pmatrix},
 \label{eq:feature-coordinate-pair}
\end{equation}
where the product equals $1$ when $d=1$.
For this pair of feature coordinates, the block of $H_{m_k}(e_1)$ and its square
contribute to the two quadratic forms the square of the scalar coefficient in
\eqref{eq:feature-coordinate-pair} times, respectively,
\begin{equation}
  2\alpha_k u_1\frac{(\alpha_k u_1)^2}{\sqrt2}
  =\sqrt2(\alpha_k u_1)^3,\qquad
  (\alpha_k u_1)^2+\prn*{\frac{(\alpha_k u_1)^2}{\sqrt2}}^2
  =(\alpha_k u_1)^2+\frac{(\alpha_k u_1)^4}{2}.
  \label{eq:swap-polynomials}
\end{equation}
Both polynomials vanish when $u_1=0$, and the first has the sign of $u_1$.
These polynomials are independent of $\beta'$.
Thus summing the contributions of the pairs $((1,\beta'),(2,\beta'))$ over all
$\beta'\in\mathbb N_0^{d-1}$ with $\abs{\beta'}\le m_k-2$ multiplies each polynomial
by the sum of the squared coefficients in \eqref{eq:feature-coordinate-pair}.

To construct $H_m(v)$ for an arbitrary unit direction $v\in\R^d$,
we first represent orthogonal transformations of $\R^d$ as orthogonal maps
of the feature space $\R^{\mathcal I_m}$.  For $j\in\{0,\ldots,m\}$ and
$u\in\R^d$, define
$f_j(u)\coloneqq\sum_{\beta\in\mathcal I_m:\,\abs{\beta}=j}u^\beta e_\beta/\sqrt{\beta!}
\in\R^{\mathcal I_m}$.  These vectors satisfy, for $u,z\in\R^d$,
\begin{equation}
 \inpr{f_j(u),f_j(z)}=\frac{\inpr{u,z}^j}{j!}
 \label{eq:feature-inner-product}
\end{equation}
and span the subspace of $\R^{\mathcal I_m}$ generated by
$\set{e_\beta:\beta\in\mathcal I_m,\ \abs{\beta}=j}$: applying a linear functional on that space
to $f_j(u)$ gives a homogeneous polynomial in $u$. If it vanishes identically,
all coefficients, and hence the functional, are zero. Every orthogonal matrix $Q\in\R^{d\times d}$
preserves the inner products \eqref{eq:feature-inner-product}, so there is a unique
orthogonal map $R_m(Q)\colon\R^{\mathcal I_m}\to\R^{\mathcal I_m}$
with $R_m(Q)f_j(u)=f_j(Qu)$ for every
degree $j\le m$ and every~$u$.  For a unit vector $v$, choose an orthogonal $Q$ with
$Qe_1=v$ and define $H_m(v)\in\R^{\mathcal I_m\times\mathcal I_m}$ by
\begin{equation}
 H_m(v)\coloneqq R_m(Q)H_m(e_1)R_m(Q)^\top.
 \label{eq:directional-operator}
\end{equation}
This \emph{orthogonal conjugation} multiplies $H_m(e_1)$ by $R_m(Q)$
on the left and by $R_m(Q)^\top$ on the right.
The definition is independent of the choice of $Q$, and reversing the direction negates the operator: $H_m(-v)=-H_m(v)$.
Both properties are proved in \cref{app:h-properties}.

For $x,y\in\B_2^d$, the \emph{comparison matrix}
$B_k(x,y)\in\R^{d_k\times d_k}$ at scale $k$ is defined by
\begin{equation}
 B_k(x,y)\coloneqq
 \begin{cases}
 \displaystyle\frac{\norm{y-x}_2}{2}
 H_{m_k}\prn*{\frac{y-x}{\norm{y-x}_2}},&x\ne y,\\
 0,&x=y.
 \end{cases}
 \label{eq:comparison}
\end{equation}
It is symmetric with $\norm{B_k(x,y)}_{\op}\le1$, and $B_k(y,x)=-B_k(x,y)$ by
$H_m(-v)=-H_m(v)$.  \Cref{app:operator} expresses $B_k(x,y)$ algebraically in $x$ and
$y$, without a choice of $Q$.

\subsection{Quadratic Forms at Utility Feature Vectors}
\label{sec:moments}

Fix an integer $k\ge0$, a utility $u\in\B_2^d$, and a unit vector
$v\in\R^d$, and write $r\coloneqq\inpr{u,v}$.
Choose an orthogonal $Q\in\R^{d\times d}$ with $Qe_1=v$.
In the orthonormal coordinates of $\R^d$ given by the columns of $Q$, the
coordinate vector of $u$ is $Q^\top u$, whose first component is $r$.
The feature definition \eqref{eq:features} and the construction of $R_{m_k}(Q)$ give
$R_{m_k}(Q)^\top\phi_k(u)=\phi_k(Q^\top u)$.
Consequently, \eqref{eq:directional-operator} yields
\begin{equation}
 \phi_k(u)^\top H_{m_k}(v)^j\phi_k(u)
 =\phi_k(Q^\top u)^\top H_{m_k}(e_1)^j\phi_k(Q^\top u)
 \quad(j\in\{1,2\}).
 \label{eq:rotated-quadratic-forms}
\end{equation}
Thus these forms equal the two polynomials in \eqref{eq:swap-polynomials}, with
$\alpha_k u_1$ replaced by $\alpha_k r$, times a common coefficient.
The coefficient multiplying each polynomial is the sum of the squared
coefficients in \eqref{eq:feature-coordinate-pair}, evaluated at $Q^\top u$, over
$\beta'\in\mathbb N_0^{d-1}$ with $\abs{\beta'}\le m_k-2$.
Since $\sum_{\ell=2}^d(Q^\top u)_\ell^2=\norm{u}_2^2-r^2$, the multinomial
identity gives
\begin{equation}
 \frac{1}{\operatorname{exp}_{m_k}(\alpha_k^2\norm{u}_2^2)}
 \sum_{\substack{\beta'\in\mathbb N_0^{d-1}\\\abs{\beta'}\le m_k-2}}
 \prod_{\ell=2}^d
 \frac{\prn{\alpha_k(Q^\top u)_\ell}^{2\beta_\ell}}{\beta_\ell!}
 =
 \frac{\operatorname{exp}_{m_k-2}\prn*{\alpha_k^2(\norm{u}_2^2-r^2)}}
      {\operatorname{exp}_{m_k}(\alpha_k^2\norm{u}_2^2)}.
 \label{eq:moment-common-factor}
\end{equation}
Replacing the two truncated sums on the right-hand side of
\eqref{eq:moment-common-factor} by exponentials
with the same arguments gives the ratio $\mathrm{e}^{-(\alpha_k r)^2}$.
Multiplying the two polynomials in \eqref{eq:swap-polynomials}, with
$\alpha_k u_1$ replaced by $z$, by $\mathrm{e}^{-z^2}$ leads to the scalar
functions, defined for $z\in\R$,
\begin{equation}
 b_1(z)\coloneqq\sqrt2\,\mathrm{e}^{-z^2}z^3,\qquad
 b_2(z)\coloneqq \mathrm{e}^{-z^2}(z^2+z^4/2).
 \label{eq:scalar-functions}
\end{equation}

\begin{restatable}{restatelemma}{momentlemma}
\label[lemma]{lem:moments}
Let $k\ge0$, $u\in\B_2^d$, and let $v\in\R^d$ be a unit vector; write
$r\coloneqq\inpr{u,v}$.  There is $\theta\in[1/2,2]$, depending on $k$, $u$, and $v$,
such that
\begin{equation}
 \phi_k(u)^\top H_{m_k}(v)\phi_k(u)=\theta\,b_1(\alpha_kr),\qquad
 \phi_k(u)^\top H_{m_k}(v)^2\phi_k(u)=\theta\,b_2(\alpha_kr).
 \label{eq:moment-bounds}
\end{equation}
\end{restatable}
\wraprestatablewithlabelrestore{momentlemma}

\noindent Since $b_1(0)=b_2(0)=0$ and $b_1(z)$ has the sign of $z$,
\eqref{eq:moment-bounds} implies that both forms vanish when $\inpr{u,v}=0$ and the
first has the sign of $\inpr{u,v}$.
The choice $m_k-2=2\alpha_k^2$ keeps the common factor $\theta$ uniformly
bounded above and below for $u\in\B_2^d$; the proof in \cref{app:moments} establishes
this bound for every unit direction $v$ by estimating the tails of the exponential series.

\section{Multiscale Matrix-Weights Algorithm}
\label{sec:algorithm}

We first describe the learner for a finite action set $Z_t=\set{x_1,\ldots,x_n}$ of
distinct actions; \cref{sec:compact} removes this restriction.  The list and its
length may change from round to round. The learner keeps one symmetric score matrix
per active scale, together with the round index, and selects its recommendation law
by solving one linear program.

Set $\eta\coloneqq2^{-9}$ and $K_t\coloneqq\ceil{2\log_2(t+1)}$.  In round $t$, the
active scales are $0,\ldots,K_t$, and every scale that becomes active in round $t$
receives the zero score matrix.  Let $W_{t,k}\in\R^{d_k\times d_k}$ denote the score
matrix of scale $k$ at the start of round $t$, after this initialization, and let
\begin{equation}
 \varrho_{t,k}\coloneqq
 \frac{\mathrm{e}^{W_{t,k}}}{\Tr \mathrm{e}^{W_{t,k}}}
 \label{eq:density}
\end{equation}
be its normalized exponential, which is positive definite and has trace one.
A positive semidefinite matrix of trace one is called a \emph{density matrix}.  The learner
evaluates the comparison matrices against these density matrices and forms
\begin{equation}
 C_{ij}\coloneqq
 \sum_{k=0}^{K_t}2^{-k}\Tr\prn*{\varrho_{t,k}B_k(x_i,x_j)},
 \qquad 1\le i,j\le n,
 \label{eq:game}
\end{equation}
a skew-symmetric matrix by $B_k(x_j,x_i)=-B_k(x_i,x_j)$.  With
$\Delta_n\coloneqq\Set*{p\in\R_{\ge0}^n}{\sum_i p_i=1}$ the probability simplex, the
recommendation probabilities are any solution of
\begin{equation}
 p\in\Delta_n,\qquad C^\top p\le0.
 \label{eq:balance}
\end{equation}
Here $C_{ij}$ compares $x_j$ against $x_i$ using the current density matrices.
The condition $C^\top p\le0$ makes this comparison nonpositive on average over
$i\sim p$ for every fixed $j$.
Such a solution exists by the minimax theorem for finite games; the proof is given in
\cref{app:balance}.
\begin{restatable}{restatelemma}{balancelemma}
\label[lemma]{lem:balance}
Every real skew-symmetric matrix $C$ admits a solution of \eqref{eq:balance}.
\end{restatable}
\wraprestatablewithlabelrestore{balancelemma}

\noindent The learner publishes $P_t\coloneqq\sum_i p_i\delta_{x_i}$, where $\delta_x$ is the
point mass at $x$, and draws $A_t$ from this law.  After observing $Y_t=y$, it updates
every active score matrix by
\begin{equation}
 W_{t+1,k}
 =W_{t,k}+
 \sum_{i=1}^np_i
 \prn*{\eta B_k(x_i,y)-\eta^2B_k(x_i,y)^2}.
 \label{eq:matrix-update}
\end{equation}
The update averages over the published law rather than using the sampled action
alone.  Consequently, for every $y=x_j$, the balance condition \eqref{eq:balance} gives
\begin{equation}
 \sum_{k=0}^{K_t}2^{-k}\sum_i p_i
       \Tr\prn*{\varrho_{t,k}B_k(x_i,y)}
 =(C^\top p)_j\le0,
 \label{eq:trace-balance}
\end{equation}
even when ties among optimal actions are broken after observing $A_t$.
Thus, under \eqref{eq:balance}, updating the active score matrices does not
increase $\sum_{k=0}^{K_t}2^{-k}\log\Tr \mathrm{e}^{W_{t,k}}$;
we prove the required one-update bound in \cref{lem:matrix-update}.

\begin{algorithm2e}[tb]
\caption{Multiscale matrix weights for finite action sets}
\label{alg:finite}
\SetAlgoNoEnd
\DontPrintSemicolon
Fix the comparison matrices in \eqref{eq:comparison} and $\eta=2^{-9}$.\;
\For{$t=1,2,\ldots$}{
 Observe $Z_t=\set{x_1,\ldots,x_n}$ and set $K_t=\ceil{2\log_2(t+1)}$.\;
 Set $W_{t,k}=0$ for newly active scales $k\le K_t$.\;
 Compute $\varrho_{t,k}$ and $C$ using \eqref{eq:density} and \eqref{eq:game}.\;
 Choose $p\in\Delta_n$ satisfying $C^\top p\le0$.\;
 Publish $P_t=\sum_i p_i\delta_{x_i}$.\;
 Draw $A_t\sim P_t$ and reveal $A_t$.\;
 Observe $Y_t$ and update every active score matrix by \eqref{eq:matrix-update}.\;
}
\end{algorithm2e}

\section{Regret Analysis}
\label{sec:analysis}

Throughout this section, fix an environment with utility $u$ whose action sets are
finite; \cref{sec:compact} treats compact action sets.

\subsection{Comparison Inequalities}

For an integer $K\ge1$, set $\eps_K\coloneqq2^{-K+1}$.
As discussed at the beginning of \cref{sec:features}, we need the first-order
contribution of \eqref{eq:matrix-update} to dominate its quadratic correction.
At each scale, \eqref{eq:moment-bounds} reduces this comparison to controlling
$b_2$ relative to $b_1$ on positive arguments. However, their ratio
$b_2(z)/b_1(z)=(1+z^2/2)/(\sqrt2\,z)$ diverges as $z\downarrow0$.
For $r\in[0,1]$, we instead compare their weighted sums over scales $0,\ldots,K$
with $r$, allowing the cutoff error $\eps_K$ in the lower bound.

\begin{restatable}{restatelemma}{dyadiclemma}
\label[lemma]{lem:dyadic}
For every $r\in[0,1]$ and $K\ge1$, the functions in
\eqref{eq:scalar-functions} satisfy
\begin{equation}
 \sum_{k=0}^K2^{-k}b_2(2^kr)\le8r,\qquad
 \frac{r-\eps_K}{8}
 \le\sum_{k=0}^K2^{-k}b_1(2^kr)\le8r.
 \label{eq:dyadic-bounds}
\end{equation}
\end{restatable}
\wraprestatablewithlabelrestore{dyadiclemma}

\noindent For a unit vector $v\in\R^d$ with $r\coloneqq\inpr{u,v}\ge0$, the functions
$b_1,b_2$ are nonnegative at $2^kr$ for every $k\ge0$.
Combining \eqref{eq:moment-bounds} with \eqref{eq:dyadic-bounds} therefore gives
\[
 \frac{r-\eps_K}{16}
 \le\sum_{k=0}^K2^{-k}\phi_k(u)^\top H_{m_k}(v)\phi_k(u)\le16r,\qquad
 0\le\sum_{k=0}^K2^{-k}\phi_k(u)^\top H_{m_k}(v)^2\phi_k(u)\le16r.
\]
For the update \eqref{eq:matrix-update}, we need a lower bound on the weighted
quadratic forms of $\eta B_k(x,y)-\eta^2B_k(x,y)^2$.
Rescaling by \eqref{eq:comparison} and treating both signs of
$\inpr{u,y-x}$ gives the following estimate.

\begin{restatable}{restatelemma}{pairlemma}
\label[lemma]{lem:pair}
For $u,x,y\in\B_2^d$ and $K\ge1$, we have
\begin{equation}
 \sum_{k=0}^K2^{-k}\phi_k(u)^\top
   \prn*{\eta B_k(x,y)-\eta^2B_k(x,y)^2}\phi_k(u)
 \ge2^{-15}\inpr{u,y-x}-2^{-6}\inpr{u,x-y}_+-2^{-13}\eps_K.
 \label{eq:pair-bound}
\end{equation}
\end{restatable}
\wraprestatablewithlabelrestore{pairlemma}

\noindent The second term on the right vanishes when $\inpr{u,y}\ge\inpr{u,x}$, in particular
when $y$ is an optimal action. In \cref{sec:compact}, we approximate the optimal
feedback action by a convex combination of listed actions and apply the lemma with
$y$ equal to each listed action; the second term need not vanish in this case. For
$\inpr{u,y-x}\ge0$, the proof also gives
$\sum_{k=0}^K2^{-k}\phi_k(u)^\top B_k(x,y)^2\phi_k(u)\le8\inpr{u,y-x}$;
this bound permits the fixed
learning rate $\eta=2^{-9}$ in \eqref{eq:pair-bound} for arbitrarily small positive
differences.  The proofs of \cref{lem:dyadic,lem:pair} are in \cref{app:analysis}.

\subsection{One Matrix Update}

For a symmetric $W$ and a unit vector $\phi$ of the same dimension, define the potential
\begin{equation}
 \Psi(W;\phi)\coloneqq\log\Tr \mathrm{e}^W-\phi^\top W\phi.
 \label{eq:potential}
\end{equation}
It is nonnegative since $\phi^\top W\phi\le\lambda_{\max}(W)\le\log\Tr \mathrm{e}^W$.

\begin{lemma}
\label[lemma]{lem:matrix-update}
For the fixed learning rate $\eta=2^{-9}$, let $W,B_1,\ldots,B_n$ be symmetric matrices of the same dimension with
$\norm{B_i}_{\op}\le1$, and let $p\in\Delta_n$.
Set $\varrho\coloneqq \mathrm{e}^W/\Tr \mathrm{e}^W$ and
$W^+\coloneqq W+\sum_i p_i(\eta B_i-\eta^2B_i^2)$.
Then we have
\begin{equation}
 \log\Tr \mathrm{e}^{W^+}-\log\Tr \mathrm{e}^W
 \le\eta\sum_i p_i\Tr(\varrho B_i).
 \label{eq:trace-update}
\end{equation}
For every unit vector $\phi$, it follows that
\begin{align}
 \sum_i p_i(\eta\phi^\top B_i\phi-\eta^2\phi^\top B_i^2\phi)
 &\le \Psi(W;\phi)-\Psi(W^+;\phi)
       +\eta\sum_i p_i\Tr(\varrho B_i).
 \label{eq:potential-update}
\end{align}
\end{lemma}
\begin{proof}
We first prove \eqref{eq:trace-update} for a single matrix, then use convexity
of the log-trace exponential to average over the matrices.
For $\abs{z}\le1/2$, the scalar inequality $\mathrm{e}^{z-z^2}\le1+z$ holds, since
$\log(1+z)-z+z^2$ vanishes at zero and has derivative $z(1+2z)/(1+z)$, which
is nonpositive for $-1/2\le z\le0$ and nonnegative for $0\le z\le1/2$.
By the spectral theorem, we have
$\mathrm{e}^{\eta B-\eta^2B^2}\preceq I+\eta B$ for $\norm{B}_{\op}\le1$, and the
Golden--Thompson inequality \citep{Golden1965,Thompson1965} gives
\[
 \Tr \mathrm{e}^{W+\eta B-\eta^2B^2}
 \le\Tr(\mathrm{e}^W\mathrm{e}^{\eta B-\eta^2B^2})
 \le\Tr \mathrm{e}^W\prn*{1+\eta\Tr(\varrho B)}.
\]
Taking logarithms and using $\log(1+r)\le r$ proves \eqref{eq:trace-update} for a
single matrix $B$.

The function $W\mapsto\log\Tr \mathrm{e}^W$ is convex.  Indeed, for an orthonormal basis
$(v_j)_j$, convexity of the scalar exponential in an eigenbasis of $W$ gives
$\sum_j\mathrm{e}^{v_j^\top Wv_j}\le\Tr \mathrm{e}^W$, with equality for an eigenbasis, so
$\log\Tr \mathrm{e}^W$ is the supremum over orthonormal bases of the functions
$\log\sum_j\mathrm{e}^{v_j^\top Wv_j}$, each of which is convex by the scalar log-sum-exp
inequality.  Convexity and the single-matrix bound now imply
\[
 \log\Tr \mathrm{e}^{W+\sum_i p_i(\eta B_i-\eta^2B_i^2)}
 \le\sum_i p_i\log\Tr \mathrm{e}^{W+\eta B_i-\eta^2B_i^2}
 \le\log\Tr \mathrm{e}^W+\eta\sum_i p_i\Tr(\varrho B_i),
\]
which is \eqref{eq:trace-update}.  Subtracting the linear term in
\eqref{eq:potential} proves \eqref{eq:potential-update}.
\end{proof}

\subsection{Regret Bound for Finite Action Sets}

At zero initialization, \eqref{eq:potential} gives $\Psi(0;\phi_k(u))=\log d_k$
for every $u$, and the weights $2^{-k}$ make the sum of these initial values finite.

\begin{restatable}{restatelemma}{initiallemma}
\label[lemma]{lem:initial}
The parameters in \eqref{eq:scales} satisfy
\begin{equation}
 \sum_{k\ge0}2^{-k}\log d_k\le36\sqrt d.
 \label{eq:initial-bound}
\end{equation}
\end{restatable}
\wraprestatablewithlabelrestore{initiallemma}

\noindent The proof, in \cref{app:analysis}, bounds $\log d_k$ by
$d\log(1+m_k/d)+m_k\log(1+d/m_k)$.
Since $m_k$ is of order $4^k$, the transition $m_k\approx d$ occurs near
$2^k\approx\sqrt d$. After multiplication by $2^{-k}$, the bounds on either side
sum to $O(\sqrt d)$.

Fix a round $t$ with its list $x_1,\ldots,x_n$, probabilities $p$, and feedback
$Y_t$. By \eqref{eq:matrix-update}, the update changes
$\sum_{k=0}^{K_t}2^{-k}\phi_k(u)^\top W_{t,k}\phi_k(u)$ by
\begin{equation}
 s_t\coloneqq\sum_{k=0}^{K_t}2^{-k}\sum_i p_i\phi_k(u)^\top
 \prn*{\eta B_k(x_i,Y_t)-\eta^2B_k(x_i,Y_t)^2}\phi_k(u).
 \label{eq:round-comparison}
\end{equation}
Since every $x_i$ lies in $Z_t$ and $Y_t$ is optimal, $\inpr{u,Y_t-x_i}\ge0$, and
averaging \eqref{eq:pair-bound} over $p$ gives
\begin{equation}
 \inpr{u,Y_t-a_t}\le2^{15}s_t+4\eps_{K_t}.
 \label{eq:round-regret}
\end{equation}

Since $K_t=\ceil{2\log_2(t+1)}$, we have
$\sum_{t\ge1}\eps_{K_t}\le2\sum_{t\ge1}(t+1)^{-2}<2$.
To bound $\sum_{t=1}^T\inpr{u,Y_t-a_t}$ using \eqref{eq:round-regret},
it remains to bound $\sum_{t=1}^Ts_t$.
Applying \eqref{eq:potential-update} to the definition \eqref{eq:round-comparison}
and using \eqref{eq:trace-balance}, we obtain
\begin{equation}
  \label{eq:telescoping}
\begin{aligned}
 \sum_{t=1}^Ts_t
 &\le\sum_{t=1}^T\sum_{k=0}^{K_t}2^{-k}
 \brc*{\Psi(W_{t,k};\phi_k(u))-\Psi(W_{t+1,k};\phi_k(u))}\\
 &=\sum_{k=0}^{K_T}2^{-k}
 \brc*{\log d_k-\Psi(W_{T+1,k};\phi_k(u))}
 \le\sum_{k=0}^{K_T}2^{-k}\log d_k\le36\sqrt d.
\end{aligned}
\end{equation}
The equality follows by telescoping from each scale's zero initialization,
where $\Psi(0;\phi_k(u))=\log d_k$; the last line uses nonnegativity of $\Psi$
and \cref{lem:initial}.
Summing \eqref{eq:round-regret} and applying \eqref{eq:telescoping} now gives
\begin{equation}
 \sum_{t=1}^T\inpr{u,Y_t-a_t}
 \le2^{15}\cdot36\sqrt d+4\sum_{t=1}^T\eps_{K_t}
 <2^{15}\cdot36\sqrt d+8.
 \label{eq:finite-regret}
\end{equation}
Inequality \eqref{eq:finite-regret} holds for every sequence of optimal feedback
actions, even when ties are broken after observing $A_t$, and \eqref{eq:mean-regret}
turns it into
$\E\brc{\Reg_T}\le2^{21}\sqrt d$ for finite action sets.

\section{Compact Action Sets}
\label{sec:compact}

For a compact action set, the learner recommends from a finite list of actions.
The balance condition \eqref{eq:balance} controls comparisons only with actions
in this list, whereas the optimal feedback action $Y_t$ may lie outside it.
We therefore approximate $Y_t$ by a convex combination of listed actions to update the score matrices.

Let $\operatorname{dist}(y,S)\coloneqq\inf_{s\in S}\norm{y-s}_2$ denote the Euclidean distance from a point $y$ to a set $S$.
Fix a round $t$ and an accuracy $\delta_t>0$. The
learner chooses $x_1,\ldots,x_n\in Z_t$ such that
\begin{equation}
 {\sup_{y\in\conv(Z_t)}}
 \operatorname{dist}\prn*{y,\conv\set{x_1,\ldots,x_n}}
 \le\frac{\delta_t}{4d}.
 \label{eq:inner-approximation}
\end{equation}
Such a list exists for every compact action set $Z_t$.  For example, set $r\coloneqq\delta_t/(4d)$,
choose a finite $r/2$-net of the unit sphere, and call the linear-optimization oracle at
each net direction.  If a net direction $q$ is within $r/2$ of a unit vector $v$, the
returned action $x\in\BR_{Z_t}(q)$ satisfies $\sigma_{Z_t}(v)-\inpr{v,x}\le r$, since all
actions have norm at most one.  Separation of a point from a compact convex set then
gives \eqref{eq:inner-approximation}.  The finite-precision implementation uses a
finite rational grid of query directions, constructed in \cref{app:rational-actions}.

The learner forms the matrix $C$ in \eqref{eq:game} on this list, publishes the resulting law,
and samples from it as in \cref{alg:finite}.  Upon observing $Y_t$, it selects
$\gamma\in\Delta_n$ by a fixed rule depending on the list and $Y_t$ such that
\begin{equation}
 \abs*{\sum_j\gamma_j(x_j)_\ell-(Y_t)_\ell}
 \le\frac{\delta_t}{2d}\quad(1\le\ell\le d),\qquad
 \bar y_t\coloneqq\sum_j\gamma_jx_j.
 \label{eq:feedback-representation}
\end{equation}
Such $\gamma$ exists by \eqref{eq:inner-approximation}, and the coordinate bounds give
$\norm{Y_t-\bar y_t}_2\le\delta_t$.
These linear constraints allow $\gamma$ to be found by solving a linear program.
The learner averages the increments in \eqref{eq:matrix-update} evaluated at
$y=x_j$, with weights $\gamma_j$:
\begin{equation}
 W_{t+1,k}=W_{t,k}
 +\sum_{i,j}p_i\gamma_j
 \prn*{\eta B_k(x_i,x_j)-\eta^2 B_k(x_i,x_j)^2}.
 \label{eq:compact-update}
\end{equation}
Since $\sum_j\gamma_j(C^\top p)_j\le0$ by
\eqref{eq:trace-balance}, \eqref{eq:trace-update} implies that
$\sum_{k=0}^{K_t}2^{-k}\log\Tr\mathrm{e}^{W_{t,k}}$ does not increase by the update
\eqref{eq:compact-update}.

To apply \eqref{eq:pair-bound}, we bound its negative-part term using the gaps
$\Delta_u(Z_t,x_j)=\sigma_{Z_t}(u)-\inpr{u,x_j}\ge0$. For all $i,j$, we have
\begin{equation}
 \inpr{u,x_i-x_j}_+\le{\Delta_u(Z_t,x_j)},\qquad
 \sum_j\gamma_j{\Delta_u(Z_t,x_j)}=\inpr{u,Y_t-\bar y_t}\le\delta_t.
 \label{eq:representation-suboptimality}
\end{equation}
For the update \eqref{eq:compact-update}, the quantity in
\eqref{eq:round-comparison} becomes
\begin{equation}
 s_t\coloneqq\sum_{k=0}^{K_t}2^{-k}\sum_{i,j}p_i\gamma_j\phi_k(u)^\top
 \prn*{\eta B_k(x_i,x_j)-\eta^2B_k(x_i,x_j)^2}\phi_k(u).
 \label{eq:compact-round-comparison}
\end{equation}
Averaging \eqref{eq:pair-bound} with these weights and using
\eqref{eq:representation-suboptimality} gives
\begin{equation}
 \inpr{u,Y_t-a_t}
 \le2^{15}s_t+(1+2^9)\delta_t+4\eps_{K_t},
 \label{eq:compact-round-regret}
\end{equation}
where $\delta_t$ accounts for replacing $Y_t$ by $\bar y_t$ and $2^9\delta_t$ for the
second term in \eqref{eq:pair-bound}.  \Cref{lem:matrix-update} applies at each scale
to the matrices $B_k(x_i,x_j)$ with the weights $p_i\gamma_j$, which are nonnegative
and sum to one, and telescoping as in \cref{sec:analysis} gives
\begin{equation}
 \sum_{t=1}^T\inpr{u,Y_t-a_t}
 \le2^{15}\cdot36\sqrt d+
     \sum_{t=1}^T\prn*{(1+2^9)\delta_t+4\eps_{K_t}}.
 \label{eq:compact-regret}
\end{equation}
Now we can prove the regret bound for compact action sets.
\regrettheorem*
\begin{proof}
Take $\delta_t\coloneqq(t+1)^{-2}$.  Since $\eps_{K_t}\le2(t+1)^{-2}$ and
$\sum_{t\ge1}(t+1)^{-2}<1$, \eqref{eq:compact-regret} and \eqref{eq:mean-regret} give
\[
 \E\brc{\Reg_T}\le2^{15}\cdot36\sqrt d+521<2^{21}\sqrt d,
\]
where the last inequality uses $d\ge1$.
\end{proof}

Under the rational oracle model of \cref{app:oracle-model}, a
finite-precision implementation achieves the same bound
$\E\brc{\Reg_T}\le2^{21}\sqrt d$.
\Cref{app:computation} constructs this implementation and bounds the total
contribution of numerical errors to regret by a universal constant, proving
\cref{thm:computation}.
As summarized in \cref{tab:comparison}, for $T\ge2$ each round $t\le T$ uses
$(dT)^{O(d)}$ linear-optimization oracle calls (\cref{app:rational-actions}), and
its remaining computation terminates (\cref{app:termination}).

\section{Conclusion and Discussion}
\label{sec:discussion}

We have presented a multiscale matrix-weights algorithm for online inverse linear optimization
with $O(\sqrt d)$ expected regret.
Its rational implementation terminates each round under the oracle
assumptions in \cref{app:oracle-model}.  Whether the same regret order can be
attained with running time polynomial in the dimension, horizon, and input
length remains open.
Another direction is to retain the $O(\sqrt d)$ dimension dependence in a regret
guarantee that is robust to suboptimal feedback.

\section*{AI Usage}
The main result was obtained through extended discussions with GPT-6 Astra in ChatGPT.
GPT-6 Astra in ChatGPT, OpenAI Codex, and Claude Fable 5.1 were also used to improve the presentation.

\section*{Acknowledgements}
Shinsaku Sakaue was supported by JST BOOST Program Japan Grant Number
\mbox{JPMJBY24D1}.

\begin{refcontext}[sorting=nyt]
\printbibliography
\end{refcontext}
\appendix
\section{Proof of the Lower Bound}
\label{app:lower-bound}

\lowerboundproposition*
\begin{proof}
Let $S_1,\ldots,S_d$ be independent uniform random signs, set
$U\coloneqq d^{-1/2}(S_1,\ldots,S_d)$, and let the environment present
$Z_i\coloneqq\set{-e_i,e_i}$ in round $i\le d$, where $e_i$ is the $i$th coordinate vector,
and $Z_1$ in every later round.  The optimal action in round $i\le d$ is $S_ie_i$, so the
feedback before round $i$ reveals only $S_1,\ldots,S_{i-1}$, and $S_i$ is independent
of the information available to the learner when $A_i$ is selected.  Hence
$\E\brc{\inpr{U,A_i}}=0$, whereas $\sigma_{Z_i}(U)=d^{-1/2}$, and summing over $i\le d$
gives $\E\brc{\Reg_d}=\sqrt d$, with expectation over both the signs and the learner's
randomness. Averaging over the signs yields a fixed value $u$ of $U$ satisfying
$\E\brc{\Reg_d}\ge\sqrt d$, now with expectation only over the learner's randomness.
Instantaneous regret is nonnegative, so the bound
persists for every $T\ge d$, and the optimal action is unique in every round.
\end{proof}

\section{Proofs for Section~\ref{sec:features}}
\label{app:features}

\subsection{\texorpdfstring{Properties of $H_m(v)$}{Properties of Hm(v)}}
\label{app:h-properties}

We prove that the construction of $H_m(v)$ in
\cref{sec:feature-construction} is independent of $Q$ and that reversing $v$
negates the matrix.
For integers $d\ge1$ and $m\ge2$, recall the feature index set
$\mathcal I_m=\Set{\beta\in\mathbb N_0^d}{\abs{\beta}\le m}$, where
$\abs{\beta}=\sum_{\ell=1}^d\beta_\ell$, and the standard basis
$(e_\beta)_{\beta\in\mathcal I_m}$ of $\R^{\mathcal I_m}$.
Write $\beta=(\beta_1,\beta')$, with
$\beta'=(\beta_2,\ldots,\beta_d)\in\mathbb N_0^{d-1}$; when $d=1$, $\beta'$
is the empty tuple.
By \eqref{eq:swap}, the matrix $H_m(e_1)\in\R^{\mathcal I_m\times\mathcal I_m}$
satisfies
\begin{equation}
 H_m(e_1)e_{(1,\beta')}=e_{(2,\beta')},\qquad
 H_m(e_1)e_{(2,\beta')}=e_{(1,\beta')}
 \quad(\abs{\beta'}\le m-2)
 \label{eq:h-pairs-recall}
\end{equation}
for every such $\beta'$ and is zero on every basis vector outside the union
of these pairs. Here $e_1\in\R^d$ is the first coordinate vector.

For an orthogonal matrix $Q\in\R^{d\times d}$, the construction in
\cref{sec:feature-construction} gives the unique orthogonal map
$R_m(Q)\colon\R^{\mathcal I_m}\to\R^{\mathcal I_m}$ satisfying, for
$j\in\{0,\ldots,m\}$ and $u\in\R^d$,
\begin{equation}
 R_m(Q)f_j(u)=f_j(Qu),\qquad
 f_j(u)=\sum_{\beta\in\mathcal I_m:\,\abs{\beta}=j}
             \frac{u^\beta}{\sqrt{\beta!}}e_\beta.
 \label{eq:feature-map-recall}
\end{equation}
Here $u^\beta=\prod_{\ell=1}^d u_\ell^{\beta_\ell}$ and
$\beta!=\prod_{\ell=1}^d\beta_\ell!$, with $u_\ell^0=1$.
For a unit vector $v\in\R^d$ and an orthogonal $Q$ with $Qe_1=v$,
\eqref{eq:directional-operator} defines
\begin{equation}
 H_m(v)=R_m(Q)H_m(e_1)R_m(Q)^\top.
 \label{eq:h-conjugation-recall}
\end{equation}

\begin{lemma}
\label{lem:directional-properties}
For integers $d\ge1$ and $m\ge2$ and a unit vector $v\in\R^d$, the matrix in
\eqref{eq:h-conjugation-recall} is independent of the orthogonal matrix $Q$
with $Qe_1=v$. It is symmetric, has operator norm one, and satisfies
$H_m(-v)=-H_m(v)$.
\end{lemma}
\begin{proof}
Let $Q_0\in\R^{d\times d}$ be orthogonal with $Q_0e_1=e_1$.
For each $\ell\in\{0,\ldots,m-2\}$, consider the two subspaces of
$\R^{\mathcal I_m}$ with bases
$(e_{(1,\beta')})_{\beta'\in\mathbb N_0^{d-1}:\,\abs{\beta'}=\ell}$ and
$(e_{(2,\beta')})_{\beta'\in\mathbb N_0^{d-1}:\,\abs{\beta'}=\ell}$,
of total degrees $\ell+1$ and $\ell+2$, respectively.
With the same ordering of $\beta'$, the restrictions of $R_m(Q_0)$ to these
subspaces have the same matrix by \eqref{eq:feature-map-recall}: $Q_0$ fixes
the first coordinate of $\R^d$ and applies the same orthogonal transformation
to its last $d-1$ coordinates.
Taken over $\ell=0,\ldots,m-2$, these subspaces span the orthogonal complement of the kernel of
$H_m(e_1)$, so the orthogonal map $R_m(Q_0)$ also preserves that kernel.
By \eqref{eq:h-pairs-recall}, $H_m(e_1)$ interchanges the corresponding basis
vectors of the two subspaces for each $\ell$, so it commutes with $R_m(Q_0)$.
Hence $R_m(Q_0)H_m(e_1)R_m(Q_0)^\top=H_m(e_1)$.
Any two choices of $Q$ with $Qe_1=v$ differ by right multiplication by such
a $Q_0$. The characterization \eqref{eq:feature-map-recall} implies
$R_m(QQ_0)=R_m(Q)R_m(Q_0)$, so \eqref{eq:h-conjugation-recall} gives the same
$H_m(v)$ for both choices.

The blocks in \eqref{eq:h-pairs-recall} are symmetric and have operator norm
one. At least one such block exists because $m\ge2$.
Since $R_m(Q)$ is orthogonal, \eqref{eq:h-conjugation-recall} shows that
$H_m(v)$ is symmetric and has operator norm one.

To reverse the direction, let
$D\coloneqq\operatorname{diag}(-1,1,\ldots,1)\in\R^{d\times d}$.
This matrix reflects $u\in\R^d$ across the hyperplane $u_1=0$:
$Du=(-u_1,u_2,\ldots,u_d)$.
By \eqref{eq:feature-map-recall}, its feature-space map satisfies
$R_m(D)e_\beta=(-1)^{\beta_1}e_\beta$ for $\beta\in\mathcal I_m$.
On each pair in \eqref{eq:h-pairs-recall}, it therefore negates
$e_{(1,\beta')}$ and fixes $e_{(2,\beta')}$.
Multiplying each block by these signs on both sides gives
\begin{equation}
 R_m(D)H_m(e_1)R_m(D)^\top=-H_m(e_1).
 \label{eq:reflection-conjugation}
\end{equation}
Since $(QD)e_1=-v$ and $R_m(QD)=R_m(Q)R_m(D)$, substituting $QD$ for $Q$ in
\eqref{eq:h-conjugation-recall} and applying \eqref{eq:reflection-conjugation}
gives $H_m(-v)=-H_m(v)$, completing the proof.
\end{proof}

\subsection{Quadratic Forms at Utility Feature Vectors}
\label{app:moments}

For the quadratic forms in \cref{sec:moments}, recall
from \eqref{eq:scalar-functions} that, for $z\in\R$,
\begin{equation}
 b_1(z)=\sqrt2\,\mathrm{e}^{-z^2}z^3,\qquad
 b_2(z)=\mathrm{e}^{-z^2}(z^2+z^4/2).
\end{equation}
The scale parameters in \eqref{eq:scales} are $\alpha_k=2^k$ and
$m_k=2\alpha_k^2+2$ for integers $k\ge0$.

\momentlemma*
\begin{proof}
We first compute both quadratic forms in orthonormal coordinates
of $\R^d$ with first basis vector $v$, then bound their common factor $\theta$.
Let $Q\in\R^{d\times d}$ be orthogonal with $Qe_1=v$, and set
$u'\coloneqq Q^\top u\in\R^d$; its first
coordinate is $r$, and its remaining coordinates form a vector $u''\in\R^{d-1}$ with
$\norm{u''}_2^2=\norm{u}_2^2-r^2$.
The orthogonal map
$R_{m_k}(Q)\colon\R^{\mathcal I_{m_k}}\to\R^{\mathcal I_{m_k}}$ satisfies
$R_{m_k}(Q)^\top\phi_k(u)=\phi_k(u')$ by \eqref{eq:features} and its construction in
\cref{sec:features}.  Thus \eqref{eq:rotated-quadratic-forms} reads
\begin{equation}
 \phi_k(u)^\top H_{m_k}(v)^j\phi_k(u)
 =\phi_k(u')^\top H_{m_k}(e_1)^j\phi_k(u')
 \quad(j\in\{1,2\}).
\end{equation}
Abbreviate $m\coloneqq m_k$,
$\alpha\coloneqq\alpha_k$, $s\coloneqq\alpha^2\norm{u}_2^2$, and
$s_\perp\coloneqq s-\alpha^2r^2=\alpha^2\norm{u''}_2^2$.
Here $\operatorname{exp}_\ell(z)=\sum_{i=0}^\ell z^i/i!$ for
$\ell\in\mathbb N_0$ and $z\in\R$, as in \cref{sec:features}.
For $\beta'\in\mathbb N_0^{d-1}$ (the empty tuple if $d=1$),
\eqref{eq:features} gives
\begin{equation}
 \phi_k(u')_{(j,\beta')}
 =\frac{\alpha^{j+\abs{\beta'}}r^j(u'')^{\beta'}}
       {\sqrt{j!\,\beta'!\,\operatorname{exp}_m(s)}}
 \qquad(j\in\set{1,2},\ \abs{\beta'}\le m-2).
\end{equation}
By \eqref{eq:h-pairs-recall}, each pair
$\set{(1,\beta'),(2,\beta')}$ with $\abs{\beta'}\le m-2$ contributes
$2\phi_k(u')_{(1,\beta')}\phi_k(u')_{(2,\beta')}$ to the first form, and since
$H_m(e_1)^2$ fixes both $e_{(1,\beta')}$ and $e_{(2,\beta')}$, the pair
contributes $\phi_k(u')_{(1,\beta')}^2+\phi_k(u')_{(2,\beta')}^2$ to the second.
Summing over $\abs{\beta'}\le m-2$ with the multinomial identity
$\sum_{\abs{\beta'}=i}(u'')^{2\beta'}/\beta'!=\norm{u''}_2^{2i}/i!$ gives
the two forms with the common coefficient on the right-hand side of
\eqref{eq:moment-common-factor}:
\begin{align}
 \phi_k(u)^\top H_m(v)\phi_k(u)
 &=\sqrt2\,(\alpha r)^3
   \frac{\operatorname{exp}_{m-2}(s_\perp)}{\operatorname{exp}_m(s)},
 \\
 \phi_k(u)^\top H_m(v)^2\phi_k(u)
 &=\prn*{(\alpha r)^2+\frac{(\alpha r)^4}{2}}
   \frac{\operatorname{exp}_{m-2}(s_\perp)}{\operatorname{exp}_m(s)}.
\end{align}
Since $\mathrm{e}^{-(\alpha r)^2}\mathrm{e}^{-s_\perp}=\mathrm{e}^{-s}$ holds,
the two right-hand sides equal $\theta\,b_1(\alpha r)$ and
$\theta\,b_2(\alpha r)$, respectively, with
\[
 \theta\coloneqq\frac{\mathrm{e}^{-s_\perp}\operatorname{exp}_{m-2}(s_\perp)}
                     {\mathrm{e}^{-s}\operatorname{exp}_m(s)}.
\]
For $z\ge0$ and a nonnegative integer $\ell\ge2z$, the exponential series gives
\[
 \mathrm{e}^z-\operatorname{exp}_\ell(z)
 =\sum_{j=\ell+1}^{\infty}\frac{z^j}{j!}
 \le\frac1{\ell+1}\sum_{j=1}^{\infty}\frac{jz^j}{j!}
 =\frac{z\mathrm{e}^z}{\ell+1}\le\frac{\mathrm{e}^z}{2}.
\]
Thus we have $\mathrm{e}^{-z}\operatorname{exp}_\ell(z)\in[1/2,1]$.
Both $s$ and $s_\perp$ lie in $[0,\alpha^2]$ and $m-2=2\alpha^2$ holds.
Applying this bound with $(z,\ell)=(s_\perp,m-2)$ and $(z,\ell)=(s,m)$ places
the numerator and the denominator of $\theta$ in $[1/2,1]$, so
$\theta\in[1/2,2]$ follows.
\end{proof}

\section{Proof for Section~\ref{sec:algorithm}}
\label{app:balance}

For $n\ge1$, recall the probability simplex
$\Delta_n=\Set*{p\in\R_{\ge0}^n}{\sum_{i=1}^n p_i=1}$.
Given $C\in\R^{n\times n}$, the condition \eqref{eq:balance} requires
$p\in\Delta_n$ with $C^\top p\le0$.

\balancelemma*
\begin{proof}
Let $C\in\R^{n\times n}$ be skew-symmetric, with $n\ge1$.
The minimax theorem for finite games gives
\[
 \min_{p\in\Delta_n}\max_{q\in\Delta_n}p^\top Cq
 =\max_{q\in\Delta_n}\min_{p\in\Delta_n}p^\top Cq\le0,
\]
since for each $q\in\Delta_n$, choosing $p=q$ bounds the inner minimum by
$q^\top Cq=0$.  The function $p\mapsto\max_{q\in\Delta_n}p^\top Cq
=\max_{1\le j\le n}(C^\top p)_j$ is continuous on the compact simplex, so it
has a minimizer.  Every minimizing $p$ satisfies $(C^\top p)_j\le0$ for
all $1\le j\le n$, as required.
\end{proof}

\section{Proofs for Section~\ref{sec:analysis}}
\label{app:analysis}

Recall the scalar functions from \eqref{eq:scalar-functions}:
\begin{equation}
 b_1(z)=\sqrt2\,\mathrm{e}^{-z^2}z^3,\qquad
 b_2(z)=\mathrm{e}^{-z^2}(z^2+z^4/2)
 \qquad(z\in\R).
\end{equation}
The cutoff error in \cref{sec:analysis} is $\eps_K=2^{-K+1}$ for integers $K\ge1$.

\dyadiclemma*
\begin{proof}
The case $r=0$ is immediate, so let $r>0$.  For $x>0$, we first establish
\begin{equation}
 \frac{b_2(x)}x\le2\min\set{x,x^{-1}},\qquad
 \frac{b_1(x)}x\le2\min\set{x,x^{-1}}.
 \label{eq:scalar-envelopes}
\end{equation}
For $0<x\le1$, we have $b_2(x)/x\le3x/2$ and $b_1(x)/x\le\sqrt2\,x$.
For $x\ge1$, the exponential series gives $\mathrm{e}^{x^2}\ge x^2+x^4/2$, so
$b_2(x)\le1$, and $x^2+x^4/2\ge\sqrt2\,x^3$ gives $b_1(x)\le1$.  These inequalities
prove~\eqref{eq:scalar-envelopes}.\looseness=-1

Set $r_k\coloneqq2^kr$, so that $2^{-k}b_i(2^kr)=r\,b_i(r_k)/r_k$ for $i\in\set{1,2}$.
Summing over $k=0,\ldots,K$, we have $\sum_{k:r_k\le1}r_k\le2$ and
$\sum_{k:r_k>1}r_k^{-1}\le2$, so
\eqref{eq:scalar-envelopes} proves both upper bounds.  For the lower bound,
nonnegativity of $b_1$ on $[0,\infty)$ suffices when $r<\eps_K$.
Otherwise choose the smallest $k\in\set{0,\ldots,K}$ with $2^kr\ge1/2$.
Such a $k$ exists since $r\ge\eps_K=2^{-K+1}$, and $r_k\in[1/2,1]$ by
minimality and $r\le1$. At that scale,
$b_1(r_k)/r_k=\sqrt2\mathrm{e}^{-r_k^2}r_k^2\ge\sqrt2/(4\mathrm{e})>1/8$, so this single term is at
least $r/8$.
\end{proof}

The following comparison bound uses the learning rate $\eta=2^{-9}$ fixed in
\cref{sec:algorithm}.

\pairlemma*
\begin{proof}
Write $L$ for the left-hand side of \eqref{eq:pair-bound}.  If $x=y$, then $L=0$
and the claim holds.  Otherwise set $h\coloneqq\norm{y-x}_2/2$,
$v\coloneqq(y-x)/\norm{y-x}_2$, and $r\coloneqq\inpr{u,v}$.
Then $0<h\le1$, $\abs{r}\le1$, and $\inpr{u,y-x}=2hr$.
By \eqref{eq:comparison} and \eqref{eq:moment-bounds}, there are
$\theta_k\in[1/2,2]$ for $0\le k\le K$ such that
\begin{equation}
 L=\sum_{k=0}^K2^{-k}\theta_k
 \brc*{\eta h\,b_1(2^kr)-\eta^2h^2\,b_2(2^kr)}.
 \label{eq:pair-scalar}
\end{equation}

For $r\ge0$, both scalar functions are nonnegative at $2^kr$.
Applying $1/2\le\theta_k\le2$ to \eqref{eq:pair-scalar} and then
\eqref{eq:dyadic-bounds} gives
\begin{equation}
 L
 \ge\frac{\eta h}{2}\sum_{k=0}^K2^{-k}b_1(2^kr)
       -2\eta^2h^2\sum_{k=0}^K2^{-k}b_2(2^kr)
 \ge\frac{\eta h}{16}(r-\eps_K)-16\eta^2h^2r
 \ge2^{-15}\inpr{u,y-x}-2^{-13}\eps_K.
\end{equation}
The last inequality uses $h^2\le h\le1$ and
$\eta/16-16\eta^2=2^{-14}$.
Since $\inpr{u,x-y}_+=0$, this proves \eqref{eq:pair-bound} when $r\ge0$.

For $r<0$, oddness of $b_1$, evenness of $b_2$, and the upper bounds in
\eqref{eq:dyadic-bounds}, applied at $\abs{r}$, give
\begin{equation}
 \begin{aligned}
 L
 &=-\sum_{k=0}^K2^{-k}\theta_k
 \brc*{\eta h\,b_1(2^k\abs{r})+\eta^2h^2\,b_2(2^k\abs{r})}\\
 &\ge-2\eta h\sum_{k=0}^K2^{-k}b_1(2^k\abs{r})
       -2\eta^2h^2\sum_{k=0}^K2^{-k}b_2(2^k\abs{r})\\
 &\ge-16(\eta h+\eta^2h^2)\abs{r}\\
 &\ge-16(\eta+\eta^2)h\abs{r}
 =-\prn*{2^{-6}+2^{-15}}\abs{\inpr{u,y-x}}.
 \end{aligned}
 \label{eq:negative-pair-bound}
\end{equation}
The first inequality uses $\theta_k\le2$ and nonnegativity of $b_1,b_2$ on
$[0,\infty)$.
The last inequality uses $h^2\le h$, and the final equality uses $\eta=2^{-9}$
and $\abs{\inpr{u,y-x}}=2h\abs{r}$.
Since $\inpr{u,y-x}=-\abs{\inpr{u,y-x}}$ and
$\inpr{u,x-y}_+=\abs{\inpr{u,y-x}}$, \eqref{eq:negative-pair-bound} implies
\eqref{eq:pair-bound} after adding the nonpositive term $-2^{-13}\eps_K$
to the right-hand side.
\end{proof}

For the initial-potential bound, recall from \eqref{eq:scales} that, for integers $k\ge0$,
\begin{equation}
 \alpha_k=2^k,\qquad m_k=2\alpha_k^2+2,\qquad d_k=\binom{d+m_k}{d}.
\end{equation}

\initiallemma*
\begin{proof}
For positive integers $d,m$, the binomial theorem applied with probabilities
$d/(d+m)$ and $m/(d+m)$ gives
\begin{equation}
 \log\binom{d+m}{d}
 \le d\log(1+m/d)+m\log(1+d/m).
 \label{eq:binomial-bound}
\end{equation}
The right-hand side increases with $m$, since its derivative is $\log(1+d/m)>0$.
{
Set $\mu_k\coloneqq2\alpha_k/\sqrt d$.  Since $m_k\le4\alpha_k^2=\mu_k^2d$
and $\alpha_k=2^k$, \eqref{eq:binomial-bound} gives
\begin{equation}
 \begin{aligned}
 2^{-k}\log d_k
 &\le2^{-k}\brc*{d\log(1+\mu_k^2)+\mu_k^2d\log(1+\mu_k^{-2})}\\
 &=2\sqrt d\brc*{\mu_k^{-1}\log(1+\mu_k^2)+\mu_k\log(1+\mu_k^{-2})}.
 \end{aligned}
 \label{eq:initial-substitution}
\end{equation}
For $0<\mu\le1$, the inequalities $\log(1+\mu^2)\le\mu^2$ and
$\log(1+\mu^{-2})\le\log2+2\log(1/\mu)$ bound the last bracketed expression in
\eqref{eq:initial-substitution} by $\mu(1+\log2+2\log(1/\mu))$.
The last bracketed expression in \eqref{eq:initial-substitution}
is unchanged under $\mu\mapsto\mu^{-1}$, so we have
}
\begin{equation}
 2^{-k}\log d_k\le2\sqrt d\,g(\mu_k),
 \quad\text{where}\quad
 g(\mu)\coloneqq
 \begin{cases}
 \mu\prn*{1+\log2+2\log(1/\mu)},&0<\mu\le1,\\
 \mu^{-1}\prn*{1+\log2+2\log\mu},&\mu>1.
 \end{cases}
 \label{eq:initial-two-ranges}
\end{equation}

To sum $g(\mu_k)$, list the values $\mu_k\le1$ in decreasing order and those
$\mu_k>1$ in increasing order, numbering each list from $j=0$.
The first list, if nonempty, starts in $(1/2,1]$, and the second starts in $(1,2]$,
because $\mu_{k+1}=2\mu_k$ and $\mu_0=2/\sqrt d\le2$.
Consequently, the $j$th value $\mu$ in either list satisfies
\begin{equation}
 \begin{aligned}
 2^{-j-1}<\mu\le2^{-j}&\qquad\text{in the first list, and}\\
 2^j<\mu\le2^{j+1}&\qquad\text{in the second list}.
 \end{aligned}
 \label{eq:initial-list-ranges}
\end{equation}
By \eqref{eq:initial-list-ranges}, the factor $\mu$ or $\mu^{-1}$ in $g$ is at most
$2^{-j}$, and the corresponding logarithm $\log(1/\mu)$ or $\log\mu$ is at most
$(j+1)\log2$.  Thus \eqref{eq:initial-two-ranges} gives
$g(\mu)\le2^{-j}(1+\log2+2(j+1)\log2)$ in either list.
Summing over both lists, with an empty list contributing zero, and using
$\sum_{j\ge0}2^{-j}=2$ and $\sum_{j\ge0}(j+1)2^{-j}=4$, we obtain
\begin{equation}
 \begin{aligned}
 \sum_{k\ge0}2^{-k}\log d_k
 &\le4\sqrt d\sum_{j\ge0}2^{-j}\prn*{1+\log2+2(j+1)\log2}\\
 &=2(4+20\log2)\sqrt d<36\sqrt d,
 \end{aligned}
\end{equation}
completing the proof.
\end{proof}

\section{Algebraic Computation of the Comparison Matrices}
\label{app:operator}

We first express the comparison matrices \eqref{eq:comparison} by rational
operations and square roots, without choosing an orthogonal matrix $Q$
in \eqref{eq:directional-operator}. For rational
actions, this expression yields a terminating procedure for approximation by
rational matrices to any prescribed positive rational accuracy.

Fix an integer $m\ge2$ and index the feature space by
$\mathcal I_m=\Set*{\beta\in\mathbb N_0^d}{\abs{\beta}\le m}$,
where $\abs{\beta}=\sum_{\ell=1}^d\beta_\ell$, as in \cref{sec:features},
with standard basis vectors $e_\beta$ of $\R^{\mathcal I_m}$; write $e_\ell$ for the
$\ell$th coordinate vector of $\R^d$.  For $1\le\ell\le d$, the lowering matrix
$\Gamma_\ell\in\R^{\mathcal I_m\times\mathcal I_m}$ decreases the $\ell$th index:
\begin{equation}
  \Gamma_\ell e_\beta\coloneqq
  \begin{cases}
    \sqrt{\beta_\ell}\,e_{\beta-e_\ell},&\beta_\ell>0,\\
    0,&\beta_\ell=0.
  \end{cases}
  \label{eq:lowering-matrices}
\end{equation}
For a unit vector $v\in\R^d$, set $\Gamma_v\coloneqq\sum_{\ell=1}^dv_\ell\Gamma_\ell$.
Applying \eqref{eq:lowering-matrices} to each basis vector $e_\beta$ in the homogeneous feature
vectors $f_j(u)=\sum_{\abs{\beta}=j}u^\beta e_\beta/\sqrt{\beta!}$ of \cref{sec:features}
gives
\begin{equation}
  \Gamma_vf_j(u)=\inpr{v,u}f_{j-1}(u)
  \qquad(1\le j\le m).
  \label{eq:lowering-homogeneous-features}
\end{equation}
Choose an orthogonal $Q\in\R^{d\times d}$ with $Qe_1=v$.
For each $j\in\{0,\ldots,m\}$, the vectors $f_j(u)$, as $u$ varies over $\R^d$,
span the subspace generated by $\set{e_\beta:\beta\in\mathcal I_m,\ \abs{\beta}=j}$.
Using $R_m(Q)f_j(u)=f_j(Qu)$ and \eqref{eq:lowering-homogeneous-features}, with
$\Gamma_v f_0(u)=0$, we obtain
\begin{equation}
 \Gamma_v=R_m(Q)\Gamma_{e_1}R_m(Q)^\top,
 \qquad
 \Gamma_v^\top\Gamma_v
 =R_m(Q)\Gamma_{e_1}^\top\Gamma_{e_1}R_m(Q)^\top.
 \label{eq:lowering-conjugation}
\end{equation}
Since $\Gamma_{e_1}^\top\Gamma_{e_1}e_\beta=\beta_1e_\beta$ for
$\beta\in\mathcal I_m$, the eigenvalues of $\Gamma_v^\top\Gamma_v$ lie in
$\set{0,\ldots,m}$.

The orthogonal projector $\Pi_v\in\R^{\mathcal I_m\times\mathcal I_m}$ onto the eigenspace of $\Gamma_v^\top\Gamma_v$ for the
eigenvalue one is therefore
\begin{equation}
  \Pi_v\coloneqq
  \prod_{\substack{0\le j\le m\\j\ne1}}
  \frac{\Gamma_v^\top\Gamma_v-jI}{1-j},
  \label{eq:degree-one-projector}
\end{equation}
since the polynomial in \eqref{eq:degree-one-projector} equals one at eigenvalue
one and zero at every other possible eigenvalue. The matrix $H_m(v)$ has the expression
\begin{equation}
  H_m(v)=
  \frac{\Gamma_v^\top\Pi_v+\Pi_v\Gamma_v}{\sqrt2}.
  \label{eq:algebraic-comparison}
\end{equation}
For $v=e_1$ and every $\beta'\in\mathbb N_0^{d-1}$ with
$\abs{\beta'}\le m-2$, the first term in the numerator of
\eqref{eq:algebraic-comparison} maps $e_{(1,\beta')}$ to
$\sqrt2\,e_{(2,\beta')}$ and the second maps $e_{(2,\beta')}$ back to
$\sqrt2\,e_{(1,\beta')}$.
Both terms vanish on basis vectors outside the union of all such pairs. The transpose
$\Gamma_{e_1}^\top$ vanishes on $e_\beta$ with $\abs{\beta}=m$, so the first summand does
not leave the feature space.  Thus \eqref{eq:algebraic-comparison} agrees with
\eqref{eq:swap} for $v=e_1$.
For general $v$, the polynomial expression \eqref{eq:degree-one-projector}
and \eqref{eq:lowering-conjugation} give
\begin{equation}
 \Pi_v=R_m(Q)\Pi_{e_1}R_m(Q)^\top.
 \label{eq:projector-conjugation}
\end{equation}
Multiplying \eqref{eq:algebraic-comparison} for $e_1$ on the left by $R_m(Q)$
and on the right by $R_m(Q)^\top$, and using \eqref{eq:directional-operator},
\eqref{eq:lowering-conjugation}, and \eqref{eq:projector-conjugation}, proves
\eqref{eq:algebraic-comparison} for $v$.

At scale $k\ge0$, the parameters in \eqref{eq:scales} are
$\alpha_k=2^k$, $m_k=2\alpha_k^2+2$, and $d_k=\binom{d+m_k}{d}$.
For $x,y\in\B_2^d$, the comparison matrix \eqref{eq:comparison} is
\begin{equation}
 B_k(x,y)=
 \begin{cases}
 \displaystyle\frac{\norm{y-x}_2}{2}
 H_{m_k}\prn*{\frac{y-x}{\norm{y-x}_2}},&x\ne y,\\
 0,&x=y.
 \end{cases}
\end{equation}

\begin{lemma}
\label[lemma]{lem:comparison-computation}
Let $x,y\in\Q^d\cap\B_2^d$, let $k\ge0$, and let $0<\epsilon\le1$ be rational.
There is an algorithm that terminates on every such input and returns a symmetric
rational matrix within $\epsilon$ of $B_k(x,y)$ in operator norm.
\end{lemma}

\begin{proof}
Equality of rational vectors can be checked exactly, and $B_k(x,x)=0$ by
\eqref{eq:comparison}, so the case $x=y$ requires no approximation.  For $x\ne y$,
$\norm{y-x}_2^2$ is a positive rational, and
\eqref{eq:lowering-matrices}--\eqref{eq:algebraic-comparison} express every entry of
$B_k(x,y)$ by finitely many rational
operations and square roots of nonnegative rationals, with all denominators nonzero:
$\norm{y-x}_2>0$, and the factors $1-j$ in \eqref{eq:degree-one-projector} have
$j\ne1$.

Use rational bisection intervals for the square roots and evaluate the finite
expressions by rational interval arithmetic.  Refine the intervals until every
denominator interval excludes zero and every output interval has width less than
$\epsilon/d_k$, where the matrix has size $d_k\times d_k$.  This procedure terminates:
the square-root intervals converge to their values, and addition, multiplication, and
division by a nonzero number preserve convergence of the enclosing intervals.  The
midpoint of each output interval therefore approximates its entry with error at most
$\epsilon/d_k$. Since the operator norm of a $d_k\times d_k$ matrix is at most
$d_k$ times its largest absolute entry, the resulting matrix has error at most
$\epsilon$ in operator norm.
Averaging it with its transpose preserves this bound and makes it exactly symmetric.
\end{proof}

\section{Rational Implementation}
\label{app:computation}

This section constructs a computable learner in a rational oracle model, bounds the
effect of its numerical errors on regret, and proves that every round terminates.

\subsection{Oracle Model and Computational Guarantee}
\label{app:oracle-model}

In the \emph{rational oracle model}, the round-$t$ linear-optimization oracle accepts
$q\in\Q^d$ and $\kappa\in\Q_{>0}$ and returns $x\in Z_t\cap\Q^d$ with
\begin{equation}
 \inpr{q,x}\ge\sigma_{Z_t}(q)-\kappa,
 \label{eq:oracle}
\end{equation}
where $\sigma_{Z_t}$ is the support function \eqref{eq:support}; an exact rational
linear-optimization oracle is sufficient.  Each oracle call terminates and returns its
output in a finite binary encoding.  The feedback action $Y_t$ has rational
coordinates; the utility need not.  These assumptions restrict the represented action
sets, not \cref{thm:regret}, which is a statement about all compact action sets.
The rest of this section establishes the following computational guarantee.

\begin{restatable}{theorem}{computationtheorem}
\label{thm:computation}
In the rational oracle model, there is a randomized learner, computable
relative to the action oracle and independent of the horizon, such that for
every $d\ge1$, every $u\in\B_2^d$, every permitted environment, and every
$T\ge1$, we have $\E\brc{\Reg_T}\le2^{21}\sqrt d$.
Each round terminates for every realization of the learner's randomness and
uses finitely many independent fair random bits.  Every call to the round-$t$
oracle is made during round $t$.
\end{restatable}
\wraprestatablewithlabelrestore{computationtheorem}

In each round, the learner constructs the action list, approximates the entries
of $C$, and computes and publishes the rounded recommendation law. It then samples
and reveals an action. After observing the feedback, it computes the coefficients
of the convex combination in \eqref{eq:feedback-representation} and updates the stored score matrices.

\subsection{Accuracy Requirements and Regret}
\label{sec:precision}

We use the parameters $\alpha_k$, $m_k$, and $d_k$ of \eqref{eq:scales}, the
learning rate $\eta=2^{-9}$ and cutoff $K_t=\ceil{2\log_2(t+1)}$ of
\cref{sec:algorithm}, and $\eps_K=2^{-K+1}$ of \cref{sec:analysis}.  The
procedure stores rational symmetric score matrices $W_{t,k}$, initialized to zero
exactly when a scale becomes active.
For the current list $x_1,\ldots,x_n\in Z_t$, the stored rational score matrices define
the exact density matrices $\varrho_{t,k}$ and entries $C_{ij}$
in \eqref{eq:density} and \eqref{eq:game}:\looseness=-1
\begin{equation}
 \varrho_{t,k}=\frac{\mathrm{e}^{W_{t,k}}}{\Tr\mathrm{e}^{W_{t,k}}}
 \quad(0\le k\le K_t),\qquad
 C_{ij}=\sum_{k=0}^{K_t}2^{-k}\Tr\prn*{\varrho_{t,k}B_k(x_i,x_j)}
 \quad(1\le i,j\le n).
\end{equation}
Then $C$ is skew-symmetric with $\abs{C_{ij}}\le2$.
The implementation computes approximations to these quantities.
Let $p\in\Delta_n$ be the probability vector that the procedure publishes
and uses both to sample an action and to update its score matrices,
let $\gamma\in\Delta_n$ specify the convex combination
$\bar y_t=\sum_j\gamma_jx_j$ in \eqref{eq:feedback-representation}, and let
$a_t=\sum_ip_ix_i$.
Let $\delta_t,\xi_t,\zeta_t>0$ be tolerances for approximating $Y_t$ by $\bar y_t$,
violation of the balance condition \eqref{eq:balance}, and error in each update of a score matrix, respectively.
Assume $\norm{Y_t-\bar y_t}_2\le\delta_t$ and
{
\begin{equation}
 \max_j (C^\top p)_j\le\xi_t,\qquad
 \norm*{W_{t+1,k}-W_{t,k}
 -\sum_{i,j}p_i\gamma_j
   \prn*{\eta B_k(x_i,x_j)-\eta^2B_k(x_i,x_j)^2}}_{\op}
 \le\zeta_t,
 \label{eq:implementation-errors}
\end{equation}%
}%
where the second inequality is required for every active scale $0\le k\le K_t$
and $B_k(x_i,x_j)$ is defined by \eqref{eq:comparison}.
A rational implementation satisfying these conditions with
$\delta_t=\xi_t=\zeta_t=(t+1)^{-2}$ is constructed in
\crefrange{app:rational-actions}{app:rational-updates}.

\begin{lemma}
\label[lemma]{lem:perturbation}
Under these conditions, every optimal feedback sequence satisfies
\begin{equation}
 \sum_{t=1}^T\inpr{u,Y_t-a_t}
 \le2^{15}\cdot36\sqrt d+
 \sum_{t=1}^T
 \prn*{(1+2^9)\delta_t+4\eps_{K_t}+2^6\xi_t+2^{17}\zeta_t}.
 \label{eq:perturbed-regret}
\end{equation}
\end{lemma}
\begin{proof}
{Recall the potential \eqref{eq:potential}, defined for a symmetric matrix
$W$ and a unit vector $\phi$ of the same dimension by
\begin{equation}
 \Psi(W;\phi)=\log\Tr\mathrm{e}^W-\phi^\top W\phi.
\end{equation}
}
If $W$ and $V$ are symmetric matrices of the same size with
$\norm{W-V}_{\op}\le r$, their ordered eigenvalues differ by at most $r$, so $\abs{\log\Tr \mathrm{e}^W-\log\Tr \mathrm{e}^V}\le r$, and $\abs{\phi^\top W\phi-\phi^\top V\phi}\le r$
for every unit vector $\phi$; hence
$\abs{\Psi(W;\phi)-\Psi(V;\phi)}\le2r$.
The error of one update in \eqref{eq:implementation-errors}
therefore adds at most $2\zeta_t$ to the right-hand side of
\eqref{eq:potential-update} at each scale, and the bound $\max_j(C^\top p)_j\le\xi_t$
adds at most $\eta\xi_t$ to the trace terms after summing over scales and averaging
over $\gamma$.  With
\[
 s_t=\sum_{k=0}^{K_t}2^{-k}\sum_{i,j}p_i\gamma_j\phi_k(u)^\top
 \prn*{\eta B_k(x_i,x_j)-\eta^2 B_k(x_i,x_j)^2}\phi_k(u)
\]
as in \eqref{eq:compact-round-comparison}, summing the perturbed potential inequalities and using
$\sum_{k\ge0}2^{-k}=2$ gives
\[
 \sum_{t=1}^Ts_t
 \le36\sqrt d+\sum_{t=1}^T(\eta\xi_t+4\zeta_t).
\]
For every round, \eqref{eq:compact-round-regret} gives
\begin{equation}
 \inpr{u,Y_t-a_t}\le2^{15}s_t+(1+2^9)\delta_t+4\eps_{K_t}.
\end{equation}
Summing over $t$ and applying the bound on $\sum_{t=1}^Ts_t$ with
$2^{15}\eta=2^6$ proves \eqref{eq:perturbed-regret}.
\end{proof}

We take
$\xi_t=\zeta_t=\delta_t\coloneqq(t+1)^{-2}$, the accuracy used in the proof of
\cref{thm:regret}.  For this choice, $\eps_{K_t}\le2\delta_t$ and
$\sum_{t\ge1}\delta_t<1$ bound the error sum in \eqref{eq:perturbed-regret} by
$1+2^9+8+2^6+2^{17}<2^{18}$.

\subsection{Oracle Calls and Approximation of Feedback}
\label{app:rational-actions}

We construct a finite action list from which every feedback action $Y_t$
can be approximated by a rational convex combination $\bar y_t$ with
$\norm{Y_t-\bar y_t}_2\le\delta_t$.
Set $M_t\coloneqq16d^2/\delta_t=16d^2(t+1)^2$.  The learner queries the oracle
\eqref{eq:oracle} at accuracy $\kappa\coloneqq\delta_t/(8d)$ for every vector of the
rational grid
\begin{equation}
  \Set*{(j_1/M_t,\ldots,j_d/M_t)}
        {j_1,\ldots,j_d\in\set{-M_t,\ldots,M_t}}
  \label{eq:rational-direction-grid}
\end{equation}
and keeps the distinct returned actions as $x_1,\ldots,x_n$.
Since each grid point is queried once, round $t$ uses exactly
$(2M_t+1)^d=(32d^2(t+1)^2+1)^d$ oracle calls.
For $1\le t\le T$ and $T\ge2$, this is $(dT)^{O(d)}$.

For every unit vector $v$, some grid vector $q$ satisfies
$\norm{q-v}_2\le\sqrt d/(2M_t)\le\delta_t/(16d)$,
and the action $x$ returned for this $q$ satisfies
\begin{equation}
  \sigma_{Z_t}(v)-\inpr{v,x}
  \le2\norm{v-q}_2+\frac{\delta_t}{8d}
  \le\frac{\delta_t}{4d},
  \label{eq:grid-support-error}
\end{equation}
by the oracle guarantee and the unit-ball bounds on $x$ and on a maximizing action.

Every $y\in Z_t$ has distance at most $\delta_t/(4d)$ from
$\conv\set{x_1,\ldots,x_n}$: otherwise, with $z$ the nearest point of this hull to $y$
and $v\coloneqq(y-z)/\norm{y-z}_2$, the nearest-point optimality condition gives
$\max_i\inpr{v,x_i}\le\inpr{v,z}$, so
$\sigma_{Z_t}(v)-\max_i\inpr{v,x_i}\ge\norm{y-z}_2>\delta_t/(4d)$, contradicting
\eqref{eq:grid-support-error}.  Since the distance to a convex set is a convex
function, the same bound holds on $\conv(Z_t)$, which is \eqref{eq:inner-approximation}.

After observing the rational feedback $Y_t\in Z_t$, the learner computes $\gamma$ from
the constraints \eqref{eq:feedback-representation}:
\begin{equation}
 \gamma\in\Delta_n,\qquad
 \abs*{\sum_{j=1}^n\gamma_j(x_j)_\ell-(Y_t)_\ell}\le\frac{\delta_t}{2d}
 \qquad(1\le\ell\le d).
\end{equation}
The distance bound ensures feasibility. A fixed deterministic algorithm for rational
linear programming \citep[Theorem~6.4.12]{GrotschelLovaszSchrijver1988}, applied to
this feasible system with zero objective, returns a rational solution $\gamma$.
The resulting convex combination
$\bar y_t=\sum_j\gamma_jx_j$ satisfies
$\norm{Y_t-\bar y_t}_2\le\delta_t/(2\sqrt d)\le\delta_t$.

\subsection{\texorpdfstring{Matrix Exponentials and Entries of $C$}{Matrix Exponentials and Entries of C}}
\label{app:computed-densities}

We compute rational approximations of the entries $C_{ij}$ using
finite Taylor sums for the matrix exponentials in $\varrho_{t,k}$.
To bound the Taylor remainders, assume inductively that the stored updates in earlier rounds
$s<t$ have operator-norm error at most $\delta_s$. Then every active score matrix at the
start of round $t$ satisfies
$\norm{W_{t,k}}_{\op}\le t+1$: the exact increment in \eqref{eq:compact-update} has
norm at most $\eta+\eta^2<2^{-8}$ and $\sum_{s\ge1}\delta_s<1$, so summing
the increments and their errors from the zero initialization proves the
bound, including for scales activated after round one.

Let $W$ be one such stored score matrix of size $d_k\times d_k$, and let
$\widetilde W\coloneqq W+(t+1)I$.  Its spectrum lies in $[0,2t+2]$, and the shift
leaves $\mathrm{e}^W/\Tr \mathrm{e}^W$ unchanged.  Let $J$ be the least integer with
$J+1\ge24(t+1)+\log_2(2^8/\delta_t)$.
Equivalently, $\delta_t2^{J+1-24(t+1)}\ge2^8$, so $J$ can be found by rational
comparisons. Let
$E_J\coloneqq\sum_{j=0}^J\widetilde W^j/j!$, which the learner evaluates in
exact rational arithmetic.  For $0\le z\le2t+2$, the scalar Taylor remainder satisfies
\[
 \mathrm{e}^z-\sum_{j=0}^J\frac{z^j}{j!}
 \le\frac{\mathrm{e}^{2t+2}(2t+2)^{J+1}}{(J+1)!}
 \le \mathrm{e}^{2t+2}\prn*{\frac{\mathrm{e}(2t+2)}{J+1}}^{J+1}
 \le2^{-(J+1)}
 \le2^{-8}\delta_t,
\]
where the third inequality uses $J+1\ge12(2t+2)$ and the last uses
$J+1\ge\log_2(2^8/\delta_t)$.  By the spectral theorem and the positivity of the
Taylor coefficients, we have
\[
  \norm{E_J-\mathrm{e}^{\widetilde W}}_{\op}\le2^{-8}\delta_t,
  \qquad E_J\succeq I.
\]

The rational density matrix is $\widehat\varrho\coloneqq E_J/\Tr E_J$.  Using
$\Tr \mathrm{e}^{\widetilde W}\ge d_k$, we obtain, for the trace norm $\norm{\cdot}_1$,
\begin{equation}
  \norm*{\widehat\varrho-\frac{\mathrm{e}^{\widetilde W}}{\Tr \mathrm{e}^{\widetilde W}}}_1
  \le\frac{\norm{E_J-\mathrm{e}^{\widetilde W}}_1+\abs{\Tr E_J-\Tr \mathrm{e}^{\widetilde W}}}
          {\Tr \mathrm{e}^{\widetilde W}}
  \le{2\norm{E_J-\mathrm{e}^{\widetilde W}}_{\op}}\le2^{-7}\delta_t,
  \label{eq:density-precision}
\end{equation}
where the first inequality adds and subtracts $E_J/\Tr \mathrm{e}^{\widetilde W}$ and uses
$\norm{E_J}_1=\Tr E_J$.
In the second inequality, the factor $d_k$ from bounding the trace norm by the
operator norm cancels against $\Tr\mathrm{e}^{\widetilde W}\ge d_k$.\looseness=-1

For each pair $i<j$ and each active $k$, \cref{lem:comparison-computation} yields a
symmetric rational $\widehat B_k(x_i,x_j)$ with
$\norm{\widehat B_k(x_i,x_j)-B_k(x_i,x_j)}_{\op}\le2^{-7}\delta_t$; set
$\widehat B_k(x_j,x_i)\coloneqq-\widehat B_k(x_i,x_j)$ and
$\widehat B_k(x_i,x_i)\coloneqq0$.  The learner computes the rational approximations to $C_{ij}$
\[
 \widehat C_{ij}\coloneqq\sum_{k=0}^{K_t}2^{-k}
 \Tr(\widehat\varrho_{t,k}\widehat B_k(x_i,x_j))
\]
exactly for $i<j$, then sets $\widehat C_{ji}\coloneqq-\widehat C_{ij}$ for $i<j$
and $\widehat C_{ii}\coloneqq0$ for every $i$, obtaining a rational skew-symmetric
matrix $\widehat C$.  Since
$\Tr(\widehat\varrho\widehat B)-\Tr(\varrho B)
=\Tr((\widehat\varrho-\varrho)B)+\Tr(\widehat\varrho(\widehat B-B))$ and
$\widehat\varrho_{t,k}$ is positive semidefinite with trace one,
\eqref{eq:density-precision} gives
\begin{equation}
  \max_{i,j}\abs{\widehat C_{ij}-C_{ij}}
  \le2(2^{-7}\delta_t+2^{-7}\delta_t)=2^{-5}\delta_t.
  \label{eq:game-coefficient-precision}
\end{equation}

\subsection{Published Law and Stored Update}
\label{app:rational-updates}

To obtain published probabilities satisfying
$\max_j(C^\top p)_j\le\delta_t$, the learner first finds
$p^{\mathrm{LP}}\in\Delta_n$ satisfying $\widehat C^\top p^{\mathrm{LP}}\le0$
by applying the rational linear-programming algorithm of \cref{app:rational-actions};
the system is feasible by
\cref{lem:balance}, since $\widehat C$ is skew-symmetric.

We round $p^{\mathrm{LP}}$ to probabilities with a common power-of-two denominator
so that exact sampling uses a fixed finite number of fair bits on this round. Let $N_t$ be the least power of two with
$N_t\ge2^6n/\delta_t$.  Rounding each $N_tp_i^{\mathrm{LP}}$ down leaves an integer deficit
$r_t\coloneqq N_t-\sum_{i=1}^n\floor{N_tp_i^{\mathrm{LP}}}$, with $0\le r_t<n$ since
$\sum_iN_tp_i^{\mathrm{LP}}=N_t$.  Add one to each of the first $r_t$ rounded coordinates and
divide by $N_t$, defining
\begin{equation}
 p_i\coloneqq
 \begin{cases}
  (\floor{N_tp_i^{\mathrm{LP}}}+1)/N_t,&1\le i\le r_t,\\
  \floor{N_tp_i^{\mathrm{LP}}}/N_t,&r_t<i\le n.
 \end{cases}
 \label{eq:rounded-probabilities}
\end{equation}
The resulting coordinates are nonnegative and sum to one.  Rounding down removes
a total mass $r_t/N_t$, and the additions restore the same mass, so the triangle
inequality gives
\begin{equation}
 \norm{p-p^{\mathrm{LP}}}_1\le\frac{2r_t}{N_t}\le\frac{2n}{N_t}\le2^{-5}\delta_t.
 \label{eq:probability-rounding-error}
\end{equation}
Using \eqref{eq:probability-rounding-error}, $\abs{C_{ij}}\le2$, and
\eqref{eq:game-coefficient-precision}, we obtain
\begin{equation}
  \max_j\sum_i p_iC_{ij}
  \le2^{-5}\delta_t+2\norm{p-{p^{\mathrm{LP}}}}_1<\delta_t,
  \label{eq:rounded-game-balance}
\end{equation}
which is the first inequality in \eqref{eq:implementation-errors} with
$\xi_t=\delta_t$.  The learner publishes this $p$, draws $A_t$ by partitioning
$\set{0,\ldots,N_t-1}$ into intervals of lengths $N_tp_i$ and reading $\log_2N_t$
independent fair bits, and uses the same $p$ in the update.

For the stored update, we approximate the increment of each score matrix within
$\delta_t$ in operator norm, as required by \eqref{eq:implementation-errors}.
We replace the comparison matrices in \eqref{eq:compact-update} by their rational
approximations:
\begin{equation}
 W_{t+1,k}\coloneqq W_{t,k}+\sum_{i,j}p_i\gamma_j
 \prn*{\eta\widehat B_k(x_i,x_j)-\eta^2\widehat B_k(x_i,x_j)^2},
 \label{eq:stored-rational-update}
\end{equation}
evaluated in exact rational arithmetic, so that the stored matrix is rational and
symmetric.  Fix an active scale $k$ and a pair $(x_i,x_j)$, and suppress
the action arguments in $B_k$ and $\widehat B_k$.
With $\epsilon\coloneqq2^{-7}\delta_t$, the bounds $\norm{B_k}_{\op}\le1$
and $\norm{\widehat B_k-B_k}_{\op}\le\epsilon$ give
\begin{equation}
  \norm{\widehat B_k^2-B_k^2}_{\op}\le(2+\epsilon)\epsilon,
  \qquad
  \norm*{\eta(\widehat B_k-B_k)
        -\eta^2(\widehat B_k^2-B_k^2)}_{\op}
       \le2\eta\epsilon<\frac{\delta_t}2,
  \label{eq:update-computation-error}
\end{equation}
using
$\widehat B_k^2-B_k^2=(\widehat B_k-B_k)\widehat B_k+B_k(\widehat B_k-B_k)$,
$\norm{\widehat B_k}_{\op}\le1+\epsilon$, $\eta=2^{-9}$, and $\epsilon\le1$.
Since $\sum_{i,j}p_i\gamma_j=1$, the stored update
\eqref{eq:stored-rational-update} differs from the exact update
\eqref{eq:compact-update} by less than $\delta_t/2$ in operator norm, which is the
second inequality in \eqref{eq:implementation-errors} with $\zeta_t=\delta_t$.
This verifies the error bound assumed inductively in
\cref{app:computed-densities}.

\subsection{Termination and Regret Bound}
\label{app:termination}

We combine the error bounds from
\crefrange{app:rational-actions}{app:rational-updates} to prove the regret guarantee
in \cref{thm:computation}, and then verify that every round terminates using finitely
many fair random bits.

\computationtheorem*
\begin{proof}
The convex combination $\bar y_t$ satisfies $\norm{Y_t-\bar y_t}_2\le\delta_t$ by
\cref{app:rational-actions}, the probabilities satisfy
\eqref{eq:rounded-game-balance}, and the stored update satisfies
\eqref{eq:update-computation-error}, so \cref{lem:perturbation} applies with
$\xi_t=\zeta_t=\delta_t=(t+1)^{-2}$.  Its bound and \eqref{eq:mean-regret} give
$\E\brc{\Reg_T}\le2^{15}\cdot36\sqrt d+2^{18}\le2^{21}\sqrt d$.

It remains to verify computability and termination.  The integer $K_t$ is the least
$K\ge0$ with $2^K\ge(t+1)^2$ and is found by successive doubling.  The grid
\eqref{eq:rational-direction-grid} is finite, and every oracle call returns a rational
action.  The approximation of each matrix $B_k(x_i,x_j)$ terminates
by \cref{lem:comparison-computation};
the Taylor degree $J$ is found by integer comparisons with powers of two, and $E_J$ and
the entries $\widehat C_{ij}$ are evaluated by finite rational arithmetic. Both rational linear
programs are feasible and are solved by a terminating deterministic algorithm.
Probability rounding uses finitely many rational operations, and sampling uses exactly
$\log_2N_t$ fair bits.

Starting from zero score matrices, \eqref{eq:stored-rational-update}
produces rational symmetric score matrices in finitely many operations, and its error bound
preserves the norm bound of \cref{app:computed-densities} inductively, so the
construction applies in every round.  The procedure uses neither the utility nor the
horizon, and it queries the round-$t$ oracle only during round $t$.
\end{proof}
\end{document}